\documentclass[sigconf]{acmart}

\usepackage{makecell}
\usepackage{amsmath,amssymb,amsfonts}
\usepackage{amsthm}
\usepackage{subcaption}
\usepackage{colortbl}
\usepackage{multirow}
\usepackage{xcolor}
\usepackage{enumitem}
\usepackage{graphicx}
\usepackage{url}
\usepackage[framemethod=tikz]{mdframed}
\definecolor{mygray}{RGB}{243,243,244}
\newmdenv[
  innertopmargin=0pt,
  backgroundcolor=mygray,
  linecolor=none,
  innerleftmargin=0pt,
  innerrightmargin=0pt,
  leftmargin=0pt
  ]{mymath}
\AtBeginDocument{%
  }

\setcopyright{acmlicensed}
\copyrightyear{2026}
\acmYear{2026}
\setcopyright{cc}
\setcctype{by}
\acmConference[WWW '26]{Proceedings of the ACM Web Conference 2026}{April 13--17, 2026}{Dubai, United Arab Emirates}
\acmBooktitle{Proceedings of the ACM Web Conference 2026 (WWW '26), April 13--17, 2026, Dubai, United Arab Emirates}
\acmPrice{}
\acmDOI{10.1145/3774904.3792435}
\acmISBN{979-8-4007-2307-0/2026/04}

\begin{document}

\title{Heterophily-Agnostic Hypergraph Neural Networks with Riemannian Local Exchanger}

\author{Li Sun}
\email{lsun@bupt.edu.cn}
\orcid{0000-0003-4562-2279}
\authornote{Corresponding Author.}
\affiliation{%
  \institution{Beijing University of Posts and Telecommunications}
  \city{Beijing}
  \country{China}
}

\author{Ming Zhang}
\email{zhangming@ncepu.edu.cn}
\affiliation{%
  \institution{North China Electric Power University}
  \city{Beijing}
  \country{China}}

\author{Wenxin Jin}
\email{jwx2265714977@outlook.com}
\affiliation{%
  \institution{North China Electric Power University}
  \city{Beijing}
  \country{China}}

\author{Zhongtian Sun}
\email{Z.Sun-256@kent.ac.uk}
\affiliation{%
 \institution{University of Kent}
 \city{Canterbury}
 \country{United Kingdom}}

\author{Zhenhao Huang}
\email{huangzhenhao@ncepu.edu.cn}
\affiliation{%
  \institution{North China Electric Power University}
  \city{Beijing}
  \country{China}}

\author{Hao Peng}
\email{penghao@buaa.edu.cn}
\affiliation{%
  \institution{Beihang University}
  \city{Beijing}
  \country{China}}

\author{Sen Su}
\email{susen@bupt.edu.cn}
\affiliation{%
  \institution{Beijing University of Posts and Telecommunications}
  \city{Beijing}
  \country{China}
}

\author{Philip Yu}
\email{psyu@uic.edu}
\affiliation{%
  \institution{University of Illinois}
  \city{Chicgao}
  \country{United States}}

\renewcommand{\shortauthors}{Li Sun et al.}

\begin{abstract}
Hypergraphs are the natural description of higher-order interactions among objects, widely applied in social network analysis, cross-modal retrieval, etc. Hypergraph Neural Networks (HGNNs) have become the dominant solution for learning on hypergraphs. Traditional HGNNs are extended from message passing graph neural networks, following the homophily assumption, and thus struggle with the prevalent heterophilic hypergraphs that call for long-range dependence modeling. Existing solutions enlarge the message flow through the hypergraph bottleneck, mitigating the oversquashing issue and capturing long-range dependence. However, they often accelerate the loss of representation distinguishability in the 
 repeated aggregations, leading to oversmoothing. This dilemma motivates  an interesting question: \textbf{Can we develop a unified mechanism that is agnostic to both homophilic and heterophilic hypergraphs?}  In this paper, we achieve the best of both worlds through the lens of Riemannian geometry, which provides the potential to adjust the message passing behavior in different regions. The key insight lies in the connection between oversquashing and hypergraph bottleneck within the framework of Riemannian manifold heat flow. Building on this, we propose the novel idea of locally adapting the bottlenecks of different subhypergraphs. The core innovation of the proposed mechanism is the design of an \textbf{adaptive local (heat) exchanger}. Specifically, it captures the rich long-range dependencies via the Robin condition, and preserves the representation distinguishability via source terms, thereby enabling heterophily-agnostic message passing with theoretical guarantees. Based on this theoretical foundation, we present a novel \textbf{H}eat-\textbf{E}xchanger with \textbf{A}daptive \textbf{L}ocality for Hypergraph Neural Network (\textbf{\texttt{HealHGNN}}), designed as a node-hyperedge bidirectional systems with linear complexity in the number of nodes and hyperedges. Extensive experiments on both homophilic and heterophilic cases show that \texttt{HealHGNN} achieves the state-of-the-art performance.
\end{abstract}

\ccsdesc[500]{Computing methodologies}
\ccsdesc[500]{Information systems~Web mining}
\ccsdesc[500]{Computing methodologies~Neural networks}

\keywords{Hypergraph Neural Network, Heterophily, Riemannian Geometry, Long-range Dependency, Manifold Heat Flow.}


\maketitle
\vspace{-0.35cm}
\section{INTRODUCTION}

Hypergraphs have emerged as a powerful tool for capturing higher-order interactions that involve more than two objects, and routinely find themselves in crucial applications, such as social network analysis and cross-modal retrieval~\cite{23survey,24survey}. Hypergraph Neural Networks (HGNNs) have become the dominant solution for learning on hypergraphs. HGNNs \cite{AAAI19-hgnn} are typically extended from the message-passing graph neural networks, such as GCN ~\cite{17gcn} and GAT ~\cite{17gat}, using hypergraph Laplacian, and thus work under the homophily assumption that connections are formed between nodes of similar labels. Also, these methods inherit a fundamental drawback of conventional message-passing mechanism: oversmoothing (OSM)~\cite{arxiv22-deephgnn}, where node representations become indistinguishable after repeated message aggregations, and oversquashing (OSQ)~\cite{logc25-SensHNN}, where signals from distant nodes are compressed.

So far, HGNNs~\cite{ICLR22-edhnn,arxiv21-hgnn-uniGCNII,NIPS23-sheafhypergnn} are primarily designed for homophilic hypergraphs, while heterophilic hypergraphs are also ubiquitous in the real-world scenarios~\cite{aaai25-hhl}. Learning on heterophilic hypergraphs necessitates the interaction between distant nodes that can be belonged to the same class. To capture long-range dependence, existing solutions~\cite{AAAI25-KHGNN,logc25-SensHNN} make effort to enlarge the message flow through the hypergraph bottleneck, where messages are often squashed. However, they often accelerate the loss of representation distinguishability during repeated aggregations, as the drawback of oversmoothing is typically inherited. This dilemma motivates an interesting question: \textbf{Can we develop a unified mechanism that is agnostic to both homophilic and heterophilic hypergraphs?}

In this paper, we achieve the best of both worlds through the lens of Riemannian geometry. Our study begins with an interesting theoretical insight: oversquashing is inherently related to hypergraph bottleneck. Concretely, within the framework of Riemannian heat flow, hypergraph message passing is given as the heat flow governed by the Laplace–Beltrami operator. The heat kernel’s spectral decomposition, together with Cheeger–Buser's Inequality, provides the fact that spectral gap $\lambda$ is connected to hypergraph bottlenecks. Then, we rigorously prove that the increase of $\lambda$ enlarges the bottleneck, promoting long-range propagation, meanwhile accelerating the decay of Dirichlet energy. That is, the increase of $\lambda$ facilitates message flow but leads to severe oversmoothing.

Besides providing a systematic understanding of the dilemma, Riemannian geometry also enables the potential for heterophily-agnostic message passing. Traditional Euclidean spaces often struggles with seeking a single, tradeoff value of $\lambda$, given a global, uniform geometry everywhere in its space. On the contrary, a general Riemannian manifold allows for diversified geometries and behaviors in different regions, and thus affords more flexible, localized control: by treating a subhypergraph as a submanifold, we can modulate its spectral gap in each submanifold to regulate information flow adaptively, allowing for the locally adjustment of different regions. Herein, we visit the concept of boundary conditions and source terms, and construct a novel \textbf{Adaptive Local Exchanger}. Specifically, it performs local bottleneck control in sprit of Robin condition, and preserves the representation distinguishability through the energy sustainment of source terms, thereby generating informative representations with theoretical guarantees.

Based on the aforementioned construction, we propose a novel \textbf{H}eat-\textbf{E}xchanger with \textbf{A}daptive \textbf{L}ocality for Hypergraph Neural Networks, referred to as \textbf{\texttt{(HealHGNN)}}, as a simple yet effective neural implementation of the new message passing mechanism with Adaptive Local Exchanger. As shown in Figure~\ref{fig:casestudy}, our approach adaptively controls bottlenecks locally and retards the Dirichlet Energy from vanishing, so that fine-grained local adjustments translate into effective global propagation. Furthermore, reflecting the two-stage message passing, we model nodes and hyperedges as coupled systems with bidirectional interactions, capturing richer high-order dependencies. Finally, we derive a clean node-wise formulation of the updating rule, allowing for a satisfactory scalability with the linear computational complexity with respect to the graph size. 
\begin{figure}[t]
\centering
    \vspace{-0.1in}
\includegraphics[width=0.95\linewidth]{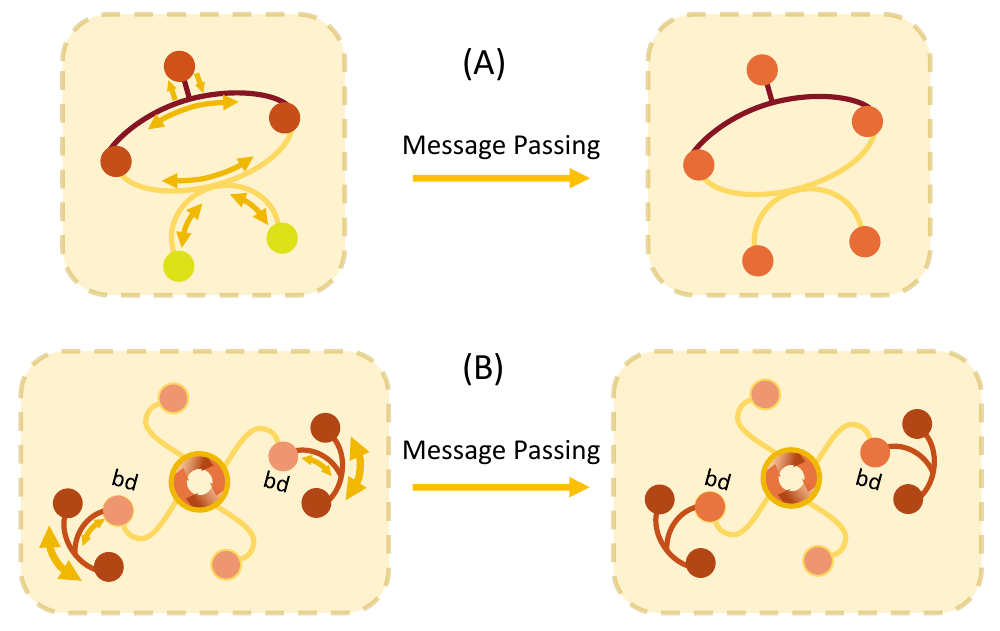}
    \vspace{-0.1in}
        \caption{(A). Traditional dynamics in HGNN. (B). Illustration of our adaptive heat exchanger; bd denotes boundary nodes.}
    \label{fig:casestudy}
    \vspace{-0.25in}
\end{figure}

In summary, key contributions of our work are listed as follows,
\begin{itemize}
    \item We rethink the heterophily of hypergraph representation learning, and introduce a unified framework of Riemannian heat flow to an interesting problem of \textbf{heterophily-agnostic hypergraph neural network design}.
    \item The key insight is that we discover the inherent connection between oversquashing and hypergraph bottlenecks through the Riemannian geometry perspective, so that we develop an \textbf{adaptive local exchanger} for the new message-passing mechanism that is agnostic to both homophilic and heterophilic hypergraphs.
    \item Based on the proposed mechanism, we propose a novel \textbf{\texttt{HealHGNN}}, a node-hyperedge bidirectional system, which achieves heterophily-agnostic message passing with linear complexity in the number of nodes plus edges and shows superior performance on a variety of real-world homophilic and heterophilic hypergraphs.
\end{itemize}

\begin{figure*}[h]
\centering
    \includegraphics[width=0.95\linewidth]{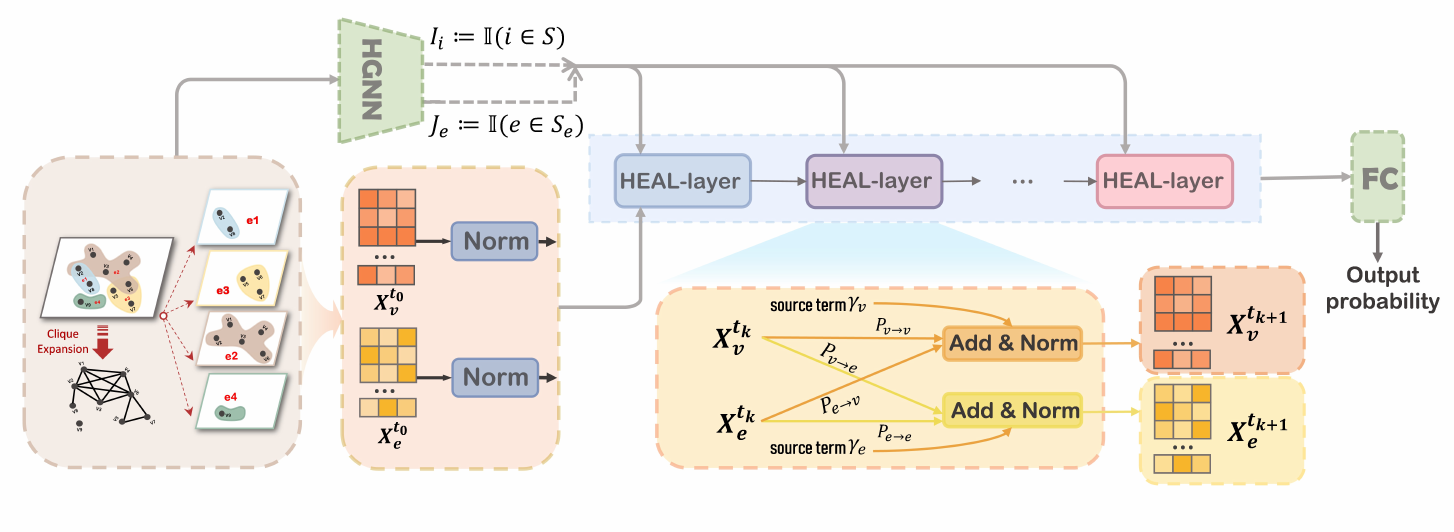}
    \vspace{-0.25in}
        \caption{Overall architecture of the proposed heat kernel based hypergraph neural network.}
    \label{fig:overall}
    \vspace{-0.1in}
\end{figure*}

\section{PRELIMINARY}
\label{headings}
This section introduces the basic concepts of hypergraphs and hypergraph learning, defines the hypergraph Laplacian, and presents two quantitative metrics: the Cheeger-constant–based bottleneck indicator for oversquashing and the Dirichlet-energy measure of global inter-node feature disparity for oversmoothing. Important notations are summarized in Appendix~\ref{notation}.

\subsubsection*{\textbf{Hypergraph.}} A hypergraph $\mathcal{G} = (\mathcal{V}, \mathcal{E}, \mathcal{H})$ generalizes the concept of a simple graph by allowing each hyperedge to connect an arbitrary subset of vertices. Here, $\mathcal{V}$ denotes the vertex set, and $\mathcal{E}$ denotes the hyperedge set. The connectivity structure can be represented by an incidence matrix $\mathcal{H} \in \{0,1\}^{|\mathcal{V}| \times |\mathcal{E}|}$, where $\mathcal{H}_{v,e}=1$ if $v \in e$, and $\mathcal{H}_{v,e} = 0$ otherwise.  Let $\mathcal{D}_{\mathcal{V}} \in \mathbb{R}_{\ge 0}^{|\mathcal{V}|\times|\mathcal{V}|}$ and $\mathcal{D}_{\mathcal{E}} \in \mathbb{R}_{\ge 0}^{|\mathcal{E}|\times|\mathcal{E}|}$ be the diagonal degree matrices of vertices and hyperedges, where $\mathcal{D}^{\mathcal{V}}_{i,i}$ counts hyperedges containing $v_i$, and $\mathcal{D}^{\mathcal{E}}_{k,k}$ counts vertices in $e_k$. 
The normalized hypergraph Laplacian
is commonly defined as $\mathcal{L} = \mathcal{I} - \mathcal{D}_{\mathcal{V}}^{-\frac12} \mathcal{H} \, \mathcal{D}_{\mathcal{E}}^{-1} \mathcal{H}^\top \mathcal{D}_{\mathcal{V}}^{-\frac12}$~\cite{AAAI19-hgnn}. We say that $\mathcal{G}$ is a $r$-graph (or $r$-uniform) if every hyperedge contains exactly r vertices.

\subsubsection*{\textbf{Cheeger Constant.}} For a hypergraph $\mathcal{G} = (\mathcal{V}, \mathcal{E}, \mathcal{H})$, the Cheeger constant $\phi(\mathcal{G})$~\cite{24-hgtheory} is defined through the following construction: for any nonempty proper vertex subset $\mathcal{S}$, its volume $\text{vol}(\mathcal{S})$ is the sum of degrees $\sum_{v \in \mathcal{S}} d(v)$, while its edge boundary $\partial_{e} \mathcal{S}$ consists of hyperedges connecting $\mathcal{S}$ and $\mathcal{V} \setminus \mathcal{S}$, counted with multiplicity: $| e \cap \mathcal{S} | \cdot | e \cap (\mathcal{V} \setminus \mathcal{S}) |$ for each $e\in \partial_{e}{\mathcal{S}}$. The boundary size $|\partial_{e} \mathcal{S}|$ sums these multiplicities over all hyperedges. The Cheeger constant $\phi(\mathcal{G})$ is then defined as:
\begin{align}
    \phi(\mathcal{G}) = \min \left\{ \frac{|\partial_{{e}} \mathcal{S}|}{\text{vol}(\mathcal{S})} \ \Bigg| \ \emptyset \neq \mathcal{S} \subset \mathcal{V}, \ \text{vol}(\mathcal{S}) \leq \text{vol}(\mathcal{V} \setminus \mathcal{S}) \right\},
\end{align}
This value measures the minimal cross-boundary connectivity strength between partitions of the hypergraph, serving as a key quantitative indicator of bottlenecks in information propagation. Furthermore, we define the boundary vertex set relative to $\mathcal{S}$ as
\begin{align}
\partial_{v} \mathcal{S} := \{\, v \in \mathcal{S} \mid 
\exists e \in \mathcal{E}:\, v \in e \wedge 
e \cap (\mathcal{V}\setminus \mathcal{S}) \neq \emptyset \,\}.
\end{align}
This set contains all vertices in $\mathcal{S}$ that belong to at least one hyperedge crossing the boundary between $\mathcal{S}$ and its complement. Correspondingly, we define the set of internal hyperedges induced by $\mathcal{S}$ as $\mathcal{S}_e := \{ e \in \mathcal{E} \mid e \subseteq \mathcal{S} \}$.

\subsubsection*{\textbf{Dirichlet Energy.}} 
Here we define the Dirichlet energy on hypergraphs following~\cite{NIPS23-sheafhypergnn}, which we use to quantify the smoothness of node embeddings with respect to the hypergraph structure. Let $X_{\mathcal{V}}^{(k)} \in \mathbb{R}^{|\mathcal{V}| \times d}$ be the node embedding matrix at the $k$-th HGNN layer. The normalized Dirichlet energy is
\begin{align}
    E(X_{\mathcal{V}}^{(k)}) &= \frac12 \sum_{e \in \mathcal{E}} \frac{1}{d_e} \sum_{u,v \in e}
\| d_v^{-1/2} x^{(k)}_v - d_u^{-1/2} x^{(k)}_u \|_2^2.
\end{align}
Intuitively, $E(X_{\mathcal{V}}^{(k)})$ measures the variation of embeddings across the nodes incident to each hyperedge: smaller energy indicates smoother (more homogeneous) features, whereas larger energy reflects stronger inter-node disparities. Consequently, this quantity serves as a natural proxy for tracking smoothing effects across layers in HGNNs.
\section{THEORETICAL INSIGHTS}
\label{others}
In heterophilous hypergraphs, informative signals must cross mixed-label regions and narrow structural bottlenecks, while increasing depth risks unwanted homogenization. To address both, we first link Cheeger Constant  to the spectral gap, quantifying the size of bottleneck. Next, by casting hypergraph propagation as a heat equation, we introduce the Robin condition to adjust bottleneck size and add source terms to prevent the vanishing of nodes disparity. Altogether, this act as an adaptive local exchanger and yields a unified, spectral view: a larger $\lambda$ improves cross-region heat exchange, yet also accelerates global smoothing—revealing a principled trade-off, all theoretically guaranteed by our analysis.
\subsection{Connection between Oversquashing and Hypergraph Bottleneck}

We consider the manifold heat flow (the functional space over Riemannian manifold), and discover that  oversquashing is inherently related to hypergraph bottleneck. 
\subsubsection{Bottleneck Size to Spectral Gap}
In a hypergraph, structural bottlenecks arise when a small set of vertices or hyperedges forms a cut that separates the vertex set into two large parts. Analogous to a narrow tube regulating heat transfer between two thermodynamic systems, any information exchange between the two sides during message passing must traverse this limited interface, thereby restricting the overall flow. The severity of the bottleneck can be quantified by the hypergraph Cheeger Constant which measures the sparsity of the smallest cut relative to the volumes of the two sides. A smaller $\phi(\mathcal{G})$ indicates a narrower “communication channel” between large regions of the hypergraph. Consequently, we have the following theorem, which establishes the relationship between $\phi(\mathcal{G})$ and the spectral gap $\lambda$.

\begin{mymath}
\begin{theorem}[\textbf{Cheeger Inequality}~\cite{24-hgtheory}]\label{cheeger}
In a connected $r$-uniform hypergraph $\mathcal{G}$, the severity of structural bottlenecks is captured by the hypergraph Cheeger constant $\phi(\mathcal{G})$. The relationship between $\phi(\mathcal{G})$ and the spectrum of the normalized hypergraph Laplacian $\mathcal{L}$ is given by the Cheeger Inequality for $r$-uniform hypergraphs:
\begin{align}
    \frac{\phi^2(\mathcal{G})}{2(r - 1)^2} \le  \lambda_{1}\le \frac{2}{r - 1} \phi(\mathcal{G}),
\end{align}
where $\lambda_{1}$ denotes the first non-zero eigenvalue of $\mathcal{L}$. 
\end{theorem}    
\end{mymath}
Thus, a small $\lambda_1$ indicates a severe bottleneck. In multi-hop aggregation schemes, these bottlenecks directly contribute to the oversquashing phenomenon.

\subsubsection{Heat Equation View of HGNNs}
Revisit feature propagation on a hypergraph $\mathcal{G} = (\mathcal{V}, \mathcal{E},\mathcal{H})$ via the normalized hypergraph Laplacian $\mathcal{L} \in \mathbb{R}^{n \times n}$. A typical layer updates is given  below.
\begin{align}
    X^{(k)} = \sigma\!\Big( \mathcal{D}_\mathcal{V}^{-1/2} \mathcal{H} \mathcal{D}_\mathcal{E}^{-1} \mathcal{H}^\top \mathcal{D}_\mathcal{V}^{-1/2}\, X^{(k-1)} \, \Theta^{(k-1)} \Big).
\end{align}
Given that the spectral gap $\lambda$ of the hypergraph Laplacian relates to information bottlenecks and oversquashing, it is natural to draw an analogy between feature propagation and heat flow on Riemannian manifold governed by the Laplace-Beltrami operator~\cite{tu2011introduction}. Within the $k$-th temporal layer interval $[t_k, t_{k+1})$, the node feature $u_k(t)$ evolves according to the heat equation
\begin{align}
     \partial_{t}u_k(t,\mathbf{x}) = -\mathcal{L}\,u_k(t,\mathbf{x}), \qquad u_k(t_k,\mathbf{x}) = Z_k,
\end{align}
where $Z_k$ is the input feature at the beginning of the $k$-th layer.
The closed-form solution is
\begin{align}
    u_k(t) = e^{-(t - t_k)\mathcal{L}} \,Z_k, \quad (t \ge t_k),
\end{align}
For sufficiently small time step $\Delta t = t_{k+1} - t_k$, the heat kernel $K(t) := e^{-t\mathcal{L}}$ admits the first-order approximation $e^{-\Delta t \mathcal{L}} \approx \mathcal{I} - \Delta t \mathcal{L},$ yielding the discrete propagation rule, where $\mathcal{G}_{\theta_k} : \mathbb{R}^n \to \mathbb{R}^n$ is a learnable mapping (e.g., a neural network).
\begin{align}
    u_k = (\mathcal{I} - \Delta t \mathcal{L})\,Z_k, 
\quad Z_k = \mathcal{G}_{\theta_k}(u_{k-1}).
\end{align}
This form unifies hypergraph neural networks, extending classical manifold heat flow theory to discrete hypergraph settings. 

\subsection{Adaptive Local Exchanger}
\subsubsection{Bottleneck Control via Robin Condition.}
Building upon the previously defined boundary node set $\partial_v \mathcal{S}$, we view inter subhypergraph information exchange as heat transfer between thermodynamic systems. The boundary nodes act as the interface (a “tube”) between systems; their propagation paradigm directly determines the effective tube width, i.e., the bottleneck size. Thus, by imposing boundary conditions, we can control the bottleneck and regulate information flux at the boundary as needed. Specifically, consider the heat equation on the hypergraph induced by node subset $\mathcal{S}$:        
\begin{align} 
	\left \{
	\begin{array}{ll}
	\partial_{t}u_\mathcal{S}(t,\mathbf{x}) = -\mathcal{L}_\mathcal{S} u_\mathcal{S}(t,\mathbf{x}),     &\mathbf{x} \in \mathcal{S} , \\
    \alpha u(\mathbf{x}) + \beta \frac{\partial u}{\partial n}(\mathbf{x}) = 0,  &\mathbf{x} \in \partial_v \mathcal{S}, \\
\end{array}
\right.
\end{align}
where $\mathcal{L_S}$ is the Laplacian restricted to $\mathcal{S}$. The boundary conditions capture realistic propagation dynamics—particularly at the interface between a subset $\mathcal{S}$ and its complement $\mathcal{V \setminus S}$. A Robin condition~\cite{robin-condition} combines a direct constraint $\alpha$ on the feature value at the boundary with a constraint $\beta$ on the directional derivative $\frac{\partial u}{\partial n}$ along the outward normal. Here, $\frac{\partial u}{\partial n}$ represents the rate at which information flows across the boundary. By tuning the coefficients $\alpha$ and $\beta$, this condition interpolates between the Dirichlet ($\beta \to 0$) and Neumann ($\alpha \to 0$) extremes, providing a flexible mechanism to control feature propagation through boundary nodes and mitigate oversquashing in bottleneck regions. 
\subsubsection{Energy Injection via Source Terms} In the presence of an external source term, the diffusion dynamics become $\partial_t u_k(t,\mathbf{x}) = -\mathcal{L}\,u_k(t,\mathbf{x}) + f(t,\mathbf{x})$. In this setting, the Dirichlet energy no longer decays exponentially as in the homogeneous case. Instead, the forcing term continually injects energy into the system, mitigating the dissipation caused by the Laplacian. 
\subsubsection{Unified Heat Exchanger.} To simultaneously address oversmoothing and oversquashing, we aim to unify the heat equation with a source term and the Robin condition into a single formulation.
Recall the problem $\partial_t u(t, \mathbf{x}) = -\mathcal{L} u(t, \mathbf{x})$ with with boundary condition $h(u(t,\mathbf{x})) = \gamma(t,\mathbf{x})$. There exists a function $w(t,\mathbf{x})$, such that $h \circ w=\gamma$ on $\partial_v\mathcal{S}$. By introducing the change of variables $v = u - w$,  the nonhomogeneous boundary condition $h(u(t,\mathbf{x})) = \gamma(t, \mathbf{x}), \mathbf{x} \in \partial_v\mathcal{S}$ can be equivalently transformed into a homogeneous form \ref{equivalence}:
\begin{equation}
\begin{cases}
\partial_t v(t, \mathbf{x}) = -\mathcal{L} v(t, \mathbf{x}) + \Gamma(w), & \mathbf{x} \in \mathcal{S},\ t>0, \\[2pt]
h(\mathbf{x}) = 0, & \mathbf{x} \in \partial_v \mathcal{S},\ t>0,
\end{cases}
\label{eq:heat-system}
\end{equation}
where $\Gamma(w) = -\mathcal{L}w(t, \mathbf{x}) - \partial_t w(t, \mathbf{x})$ and $v(0, \mathbf{x}) = Z(\mathbf{x}) - w(0,\mathbf{x})$.
\noindent\textbf{Intuitive Explanation.} 
In our adaptive local heat exchanger, the Robin condition regulates cross-region heat flow like a tunable valve, while the source term  serves as a booster that injects energy into the system. Together, they tighten or loosen local bottlenecks and slow the decay of Dirichlet energy, so that fine-grained local adjustments translate into effective global propagation.

\subsection{Guarantees on Heterophily-agnostic Message Passing}
\subsubsection{Dirichlet Energy to Spectral Gap}
From a physical view, the second law of thermodynamics states that in an isolated thermodynamic system, heat distribution irreversibly evolves toward equilibrium. This process mirrors  the mechanism of oversmoothing in deep HGNNs, where layer-wise diffusion gradually homogenizes node features and rapidly decays Dirichlet energy, flattening feature differences across hyperedges.

\begin{mymath}
\begin{theorem}[\textbf{Spectral Gap and Dirichlet Energy, \ref{prf:Spectral Gap and Dirichlet Energy}}]
\label{smooth}Let $0 = \lambda_0 < \lambda_1 \le \lambda_2 \le \dots$ be the eigenvalues of the hypergraph Laplacian $\mathcal{L}$, then the Dirichlet Energy satisfies
\begin{align}
    E(u_k(t)) 
\le e^{-2\lambda_1 (t - t_k)} E(Z_k).
\end{align}
In particular, the Dirichlet energy decays at least exponentially with rate $2\lambda_1$.
\end{theorem}    
\end{mymath}
Consistent with Theorem~\ref{cheeger}, oversmoothing can also be characterized by the spectral gap: a larger $\lambda_1$ accelerates the exponential decay of Dirichlet energy, thereby intensifying smoothing. At the same time, a larger $\lambda$ helps mitigate oversquashing. Taken together, these effects provide a theoretical validation of the intrinsic trade-off between OSQ and OSM.

\subsubsection{Guarantees of Our Constructions.} These boundary conditions modify the spectral properties of the restricted Laplacian $\mathcal{L}_S$, where locally adjusting $\lambda$ controls the amount of messages that flow out of the bottlenecks. Specifically, Dirichlet fixes boundary values and maximizes exchange, Neumann blocks exchange with zero normal flux, while Robin provides a tunable finite-conductance interface that interpolates between them. 
In the presence of an external source term, the diffusion dynamics become $\partial_t u_k(t,\mathbf{x}) = -\mathcal{L}\,u_k(t,\mathbf{x}) + f(t,\mathbf{x})$. In this setting, the Dirichlet energy no longer decays exponentially as in the homogeneous case. Instead, the forcing term continually injects energy into the system, mitigating the dissipation caused by the Laplacian. 
\begin{mymath}
\begin{theorem}[\textbf{Flexibility of the Robin Condition, \ref{prf:flexibility of rc}}]
\label{bdcondition}
Consider a thin tube–like region $S$ of length $L$, whose cross–section at position $s \in [0,L]$ is a small ball of radius $\varepsilon(s)$. The boundary of this region is given by $\partial \mathcal{S} = \{0,L\} \times \partial \mathcal{B}_{\varepsilon(s)}$. Let $\lambda_D$, $\lambda_N$, and $\lambda_{R}$ denote the spectral gaps of the Laplace operator under Dirichlet, Neumann, and Robin boundary conditions, respectively.
When the radius is constant, i.e., $\varepsilon(s) = \varepsilon_0$, the spectral gaps satisfy the ordering
$\lambda_D \;\geq\; \lambda_{R} \;\geq\; \lambda_N$.
\end{theorem}    
\end{mymath}
Furthermore, the source term selectively preserves initial information and amplifies disparity; together with the Robin boundary, it yields a heterophily-agnostic propagation mechanism.
\begin{mymath}
\begin{theorem}[\textbf{Mode Preservation under Source Energy Injection, \ref{prf:source}}]
\label{source}Consider the diffusion equation with source term $\partial_t u_k(t,\mathbf{x}) = -\mathcal{L}\,u_k(t,\mathbf{x}) + f(t,\mathbf{x})$. Expanding in the Laplacian eigenbasis gives $u(t) = \sum_i \hat u_i(t)\psi_i$, $f(t) = \sum_i \hat f_i(t)\psi_i$ with coefficients $\hat u_i(t)=\langle u(t), \psi_i \rangle$, $\hat f_i(t) = \langle f(t), \psi_i \rangle$, and each mode satisfies $\dot{\hat u}_i(t) = -\lambda_i \hat u_i(t) + \hat f_i(t)$. If the source persistently excites mode $i$ over a time step $\Delta t = t_{k+1} - t_k$, i.e., $\hat f_i(s) \ge \eta_i > 0$ for all $s \in [t_k, t_{k+1}]$, then its amplitude is bounded below by
\begin{equation}
 |\hat u_i(t_{k+1})| \ge \frac{\eta_i}{\lambda_i} (1 - e^{-\lambda_i \Delta t}) - e^{-\lambda_i \Delta t} |\hat u_i(t_k)|.   
\end{equation}
showing that persistent energy injection prevents the mode from vanishing.
\end{theorem}    
\end{mymath}
\section{HEALHGNN: A UNIFIED APPROACH}
We first discretize the hypergraph heat diffusion system with boundary conditions and source terms to a layer-wise update form. Let $X_{\mathcal{V}}(k)$ and $X_{\mathcal{E}}(k)$ denote the vertex and hyperedge representations at the k-th layer, respectively,
\begin{align}
\underbrace{
\begin{bmatrix}
X_\mathcal{V}(k+1) \\
X_\mathcal{E}(k+1)
\end{bmatrix}}_{\text{updated state}}
=
\underbrace{
f\Big(
\begin{bmatrix}
X_\mathcal{V}(k) \\
X_\mathcal{E}(k)
\end{bmatrix}\Big)}_{\text{diffusion process}},
\quad
\underbrace{
\begin{bmatrix}
X_\mathcal{V}(0) \\
X_\mathcal{E}(0)
\end{bmatrix}}_{\text{initial state }}
=
\begin{bmatrix}
Z_\mathcal{V} \\
Z_\mathcal{E}
\end{bmatrix}.
\end{align}
\subsection{Node-Hyperedge Bidirectional System}
Inspired by the continuous setting, we impose the nonhomogeneous Robin condition on hypergraph in a form directly mirroring its PDE counterpart:
\begin{equation}
    \alpha_i\,\mathbf{x}_i 
    + \beta_i\,\partial_n^{\mathrm{disc}}\mathbf{x}_i
    = \gamma_i, 
    \quad i \in \partial_v \mathcal{S},
    \label{eq.robin-graph}
\end{equation}
where $\alpha_i = \alpha(\mathbf{x}_i)$, $\beta_i = \beta(\mathbf{x}_i)$, and $\gamma_i = \gamma(\mathbf{x}_i)$ are coefficient functions depending on the current node states. 
The discrete normal derivative $\partial_n^{\mathrm{disc}}\mathbf{x}_i$ is defined as
$\partial_n^{\mathrm{disc}} \mathbf{x}_i 
:= \sum_{j \sim i,\, j \in \mathcal{V} \setminus \partial_v \mathcal{S}} 
\big(\mathbf{x}_i - \mathbf{x}_j\big)$, representing the net feature flux from boundary node $i$ to its interior neighbors.  Subject to Eq.~(\ref{eq.robin-graph}) on $\partial_v \mathcal{S}$, the heat equation in Eq.~(\ref{eq:heat-system}) governs the evolution on the interior nodes, ensuring a nontrivial steady state. The update of node embeddings on hypergraphs can be written in block-matrix form as
\vspace{-3pt} 
\begin{equation}
\label{int-bd-func}
\begin{cases}
\text{Interior:} & \mathcal{L}^{S\to S} X_{S} + \mathcal{L}^{\partial S \to S} X_{\partial S} = F_{S},\\
\text{Boundary:}&\mathrm{diag}(\alpha) X_{\partial S} + \mathrm{diag}(\beta) \mathcal{L}^{S\to\partial S} X_{S} = \Gamma_{\partial S}.
\end{cases}
\end{equation}
\vspace{-0.2cm} \\
Here, $X_S$ and $X_{\partial S}$ denote the feature matrix of interior and boundary nodes, respectively. The matrices $\mathcal{L}^{S \to S}$ and $\mathcal{L}^{\partial S \to S}$ correspond to the discrete Laplacian restricted to interior and boundary contributions. For nodes, the Laplacian is defined as $\mathcal{L}_\mathcal{V} = \mathcal{I} - \mathcal{D}_\mathcal{V}^{-1/2} \mathcal{H} \mathcal{D}_\mathcal{E}^{-1} \mathcal{H}^\top \mathcal{D}_\mathcal{V}^{-1/2}.$ Analogously, for the two-stage propagation process involving hyperedge features, we define the hyperedge Laplacian $\mathcal{L}_\mathcal{E} = \mathcal{I} - \mathcal{D}_\mathcal{E}^{-1/2} \mathcal{H}^\top \mathcal{D}_\mathcal{V}^{-1} \mathcal{H} \mathcal{D}_\mathcal{E}^{-1/2}.$ With this definition, the same block-matrix update can be applied to hyperedges. This unified formulation Eq.~(\ref{int-bd-func}) describes feature diffusion for both nodes and hyperedges, incorporating interior-boundary interactions and weighted contributions through $\mathrm{diag}(\alpha)$ and $\mathrm{diag}(\beta)$.

Building upon the separate propagation dynamics for nodes and hyperedges, we introduce bidirectional coupling between the two domains via the propagation operators $\mathcal{P}_{\mathcal{V} \to \mathcal{E}} = \mathcal{D}_\mathcal{E}^{-1/2} \mathcal{H}^\top \mathcal{D}_\mathcal{V}^{-1/2}$ and $\mathcal{P}_{\mathcal{E} \to \mathcal{V}} = \mathcal{D}_\mathcal{V}^{-1/2} \mathcal{H} \mathcal{D}_\mathcal{E}^{-1/2}$. Specifically, interior and boundary nodes in one domain receive feature contributions from the other domain, modulated by the coupling coefficients $\lambda$ and $\mu$. This design allows information to propagate not only within each domain through the Laplacian operators but also across domains, thereby capturing richer topological dependencies and enhancing interactions between vertex-level and hyperedge-level representations.

\subsection{A Clean Node-wise Formulation}
Formally, the coupled update rules can be expressed as a compact block-matrix form, which is convenient for unified analysis and computation.
\begin{align}
A
\begin{bmatrix}
X_V\\[2pt] X_E
\end{bmatrix}
=
\begin{bmatrix}
\Gamma_V\\[2pt] \Gamma_E
\end{bmatrix},\quad  
X=\begin{bmatrix}
X_{S}\\[2pt] X_{\partial{S}} \\
\end{bmatrix},\quad
\Gamma =\begin{bmatrix}
F_{S}\\[2pt] \Gamma_{\partial{S}}
\end{bmatrix}.
\end{align}
\begin{equation}
A=
\scalebox{0.75}{
$
\begin{bmatrix}
\mathcal{L}_{\mathcal{V}}^{S\to S} & \mathcal{L}_{\mathcal{V}}^{\partial_v S\to S} &-\lambda P^{S_e\to S}_{\mathcal{E}\to \mathcal{V}} & -\lambda P^{\partial_e S\to S}_{\mathcal{E}\to \mathcal{V}} \\[7pt]
\mathrm{diag}(\beta_v)\mathcal{L}_{\mathcal{V}}^{S\to \partial_v S} & \mathrm{diag}(\alpha_v) & -\lambda P^{S_e\to \partial_v S}_{\mathcal{E}\to \mathcal{V}} &-\lambda P^{\partial_e S\to \partial_v S}_{\mathcal{E}\to \mathcal{V}} \\[7pt]
-\mu P^{S\to S_e}_{\mathcal{V}\to \mathcal{E}} & -\mu P^{\partial_v S\to S_e}_{\mathcal{V}\to \mathcal{E}} & \mathcal{L}_{\mathcal{E}}^{S_e\to S_e} & \mathcal{L}_{\mathcal{E}}^{\partial_e S\to S_{e}} \\[7pt]
-\mu P^{S\to \partial_e S}_{\mathcal{V}\to \mathcal{E}} & -\mu P^{\partial_v S\to \partial_e S}_{\mathcal{V}\to \mathcal{E}} & \mathrm{diag}(\beta_e)\mathcal{L}_{\mathcal{E}}^{S_e \to \partial_e S} & \mathrm{diag}(\alpha_e)
\end{bmatrix}
$
}
\end{equation}
Let $\mathbf{D}$, $\mathbf{U}$, and $\mathbf{V}$ denote the diagonal, upper, and lower blocks of the system matrix $\mathbf{A}$, respectively. Following the Jacobi iteration scheme, we parameterize the problem-dependent coefficients via learnable mappings
\begin{align}
\scalebox{0.85}{
$
\operatorname{diag}(\alpha) = \alpha(\mathbf{X}^{(k)}; \theta_{\alpha}),\quad
\operatorname{diag}(\beta) = \beta(\mathbf{X}^{(k)}; \theta_{\beta}),\quad
\mathbf{\Gamma} = \gamma(\mathbf{X}^{(k)}; \theta_{\gamma}).
$}    
\end{align}
The update rule is then augmented with a nonlinearity $\sigma(\cdot)$ and a learnable transformation $\varphi^{(k+1)}(\cdot)$ implemented as a multi-layer perceptron (MLP), yielding
\begin{align}
\mathbf{X}^{(k+1)} = \sigma\!\left(\mathbf{D}^{-1}(\mathbf{U} + \mathbf{V})\,\varphi^{(k+1)}(\mathbf{X}^{(k)})\right)+\mathbf{D}^{-1}\mathbf{\Gamma}.    
\end{align}
This formulation inherits the convergence structure of the Jacobi method while allowing adaptive, data-driven refinement through the learnable parameters $\theta_{\alpha}, \theta_{\beta}, \theta_{\gamma}$ and $\varphi^{(k+1)}$. Let $S_{ij} := \sum_{e:\, i,j \in e} \frac{1}{d_e}$ denote the total inverse degree of hyperedges shared by nodes $i$ and $j$ and $p_i=\frac{\beta_v}{\alpha_v}$. We define the indicator function $I_i := \mathbb{I}(i \in \mathcal{S})$, which equals $1$ if node $i \in \mathcal{S}$ and $0$ otherwise. The feature update for node $i$ at iteration $k+1$ is given by
\begin{align}\scalebox{0.95}{
$
x_i^{(k+1)}
= \underbrace{\sum_{j:\, j \sim i} \mathcal{P}_{\mathcal{V} \to \mathcal{V}}(i,j) \, x_j^{(k)}}_{\text{node-to-node aggregation}}
+ \underbrace{\sum_{e:\, e \sim i} \mathcal{P}_{\mathcal{E} \to \mathcal{V}}(i,e) \, x_e^{(k)}}_{\text{edge-to-node aggregation}}
+ \gamma_i, 
$}   
\end{align}
where $j \sim i$ denotes that $i$ and $j$ share at least one hyperedge, and $e \sim i$ denotes that node $i$ belongs to hyperedge $e$. Specifically, the node-to-node propagation operator is defined as
\begin{equation}
\mathcal{P}_{\mathcal{V} \to \mathcal{V}}(i,j)
= \frac{1}{\sqrt{d_i d_j}} \, S_{ij} \,
\Big( I_i I_j \,+\, I_i \,+\, p_i (1-I_i) I_j \Big),     
\end{equation}
The edge-to-node propagation operator  is
\begin{align}
\mathcal{P}_{\mathcal{E} \to \mathcal{V}}(i,e)
= \mu \; \frac{1}{\sqrt{d_i d_e}}
\Big( I_i \,+\, \frac{1}{\alpha_v} (1-I_i) \Big),    
\end{align}
where $\mu$ is a scaling coefficient, and $\alpha_v$ controls the relative influence from hyperedges to boundary versus interior nodes.

For hyperedges, the update follows a similar form with appropriately defined edge-to-edge and node-to-edge propagation operators. We omit the detailed expressions here for brevity; the full derivation and definitions are provided in \ref{solution}, and the complete architecture of \texttt{HealHGNN} is shown in Figure~\ref{fig:overall}.

\begin{table*}[ht]
\centering
\caption{Node classification on homophilic and heterophilic hypergraphs. Top three models are coloured by \textcolor{red}{First}, \textcolor{blue}{Second}, \textcolor{orange}{Third}.}
\label{tab:nctask}
\vspace{-0.2cm}
\resizebox{0.9\textwidth}{!}{
\begin{tabular}{c|cccc|cccc}
\toprule
Type & \multicolumn{4}{c|}{Homophilic hypergraphs} & \multicolumn{4}{c}{Heterophilic hypergraphs}\\
\cmidrule(lr){2-5} \cmidrule(lr){6-9}
Dataset & DBLP-CA & Citeseer & Pubmed & NTU2012 & Congress & House & Senate & Walmart \\
\hline
HGNN \cite{AAAI19-hgnn} & 91.03 {\scriptsize$\pm$0.20} & 62.45 {\scriptsize$\pm$1.16} & 86.44 {\scriptsize$\pm$0.44} & 87.72 {\scriptsize$\pm$1.35} & 91.26 {\scriptsize$\pm$1.15} & 61.39 {\scriptsize$\pm$2.96} & 48.59 {\scriptsize$\pm$4.52} & 62.00 {\scriptsize$\pm$0.24} \\
HyperGCN \cite{nips19-hypergcn} & 89.98 {\scriptsize$\pm$0.43} & 72.76 {\scriptsize$\pm$1.12} & 82.84 {\scriptsize$\pm$0.67} & 86.36 {\scriptsize$\pm$1.86} & 87.43 {\scriptsize$\pm$1.46} & 60.13 {\scriptsize$\pm$1.76} & 50.13 {\scriptsize$\pm$3.21} & 48.74 {\scriptsize$\pm$2.81} \\
HNHN \cite{arxiv20-hnhn} & 86.78 {\scriptsize$\pm$0.29} & 72.64 {\scriptsize$\pm$1.57} & 84.21 {\scriptsize$\pm$0.64} & 86.13 {\scriptsize$\pm$0.76} & 90.89 {\scriptsize$\pm$0.31} & 67.80 {\scriptsize$\pm$2.59} & 50.93 {\scriptsize$\pm$6.33} & 51.18 {\scriptsize$\pm$0.35} \\
HCHA \cite{PR21-hcha} & 90.92 {\scriptsize$\pm$0.22} & 72.42 {\scriptsize$\pm$1.42} & 84.56 {\scriptsize$\pm$0.31} & 87.12 {\scriptsize$\pm$1.23} & 90.43 {\scriptsize$\pm$1.20} & 61.36 {\scriptsize$\pm$2.53} & 48.62 {\scriptsize$\pm$4.41} & 62.35 {\scriptsize$\pm$0.26} \\
UniGCNII \cite{arxiv21-hgnn-uniGCNII} & 91.69 {\scriptsize$\pm$0.19} & 73.05 {\scriptsize$\pm$2.21} & 87.31 {\scriptsize$\pm$0.45} & 89.30 {\scriptsize$\pm$1.33} & 92.43 {\scriptsize$\pm$0.71} & 67.25 {\scriptsize$\pm$2.57} & 49.30 {\scriptsize$\pm$4.25} & 54.45 {\scriptsize$\pm$0.37} \\
AllDeepSets \cite{iclr22-alldeepsets-allsettrans} & 91.27 {\scriptsize$\pm$0.27} & 70.83 {\scriptsize$\pm$1.63} & 87.04 {\scriptsize$\pm$0.52} & 88.69 {\scriptsize$\pm$1.24} & 91.80 {\scriptsize$\pm$1.53} & 67.82 {\scriptsize$\pm$2.40} & 48.17 {\scriptsize$\pm$5.67} & 64.55 {\scriptsize$\pm$0.33} \\
AllSetTransformer \cite{iclr22-alldeepsets-allsettrans} & 91.53 {\scriptsize$\pm$0.23} & 73.08 {\scriptsize$\pm$1.20} & 87.66 {\scriptsize$\pm$0.43} & 88.43 {\scriptsize$\pm$1.43} & 92.16 {\scriptsize$\pm$1.05} & 69.83 {\scriptsize$\pm$5.22} & \textcolor{orange}{69.33 {\scriptsize$\pm$2.20}} & \textcolor{orange}{65.46 {\scriptsize$\pm$0.25}} \\
\hline
Deep-HGNN \cite{arxiv22-deephgnn} & 91.76 {\scriptsize$\pm$0.28} & 74.07 {\scriptsize$\pm$1.64} & 87.42 {\scriptsize$\pm$0.22} & 87.12 {\scriptsize$\pm$1.13} & 92.47 {\scriptsize$\pm$1.06} & \textcolor{blue}{75.26 {\scriptsize$\pm$1.76}} & 68.39 {\scriptsize$\pm$4.79} & 58.64 {\scriptsize$\pm$0.74} \\
ED-HNN \cite{ICLR22-edhnn} & \textcolor{blue}{91.90 {\scriptsize$\pm$0.19}} & 73.70 {\scriptsize$\pm$1.38} & 87.58 {\scriptsize$\pm$0.65} & 88.07 {\scriptsize$\pm$1.28} & \textcolor{blue}{93.28 {\scriptsize$\pm$0.59}} & 72.45 {\scriptsize$\pm$2.28} & 64.79 {\scriptsize$\pm$5.14} & \textcolor{blue}{66.79 {\scriptsize$\pm$0.41}} \\
SheafHyperGNN \cite{NIPS23-sheafhypergnn} & 91.59 {\scriptsize$\pm$0.24} & 74.71 {\scriptsize$\pm$1.23} & 87.68 {\scriptsize$\pm$0.60} & - & 91.81 {\scriptsize$\pm$1.60} & 73.84 {\scriptsize$\pm$2.30} & 68.73 {\scriptsize$\pm$4.68} & - \\
PhenomNN \cite{ICML24-PhenomNN} & \textcolor{orange}{91.86 {\scriptsize$\pm$0.22}} & 74.45 {\scriptsize$\pm$0.96} & 78.12 {\scriptsize$\pm$0.24} & 85.40 {\scriptsize$\pm$0.42} & 88.24 {\scriptsize$\pm$1.47} & 70.71 {\scriptsize$\pm$2.35} & 67.70 {\scriptsize$\pm$5.24} & 63.43 {\scriptsize$\pm$0.43} \\
\hline
FrameHGNN \cite{AAAI25-FrameHGNN} & - & \textcolor{orange}{74.72 {\scriptsize$\pm$2.10}} & \textcolor{blue}{88.73 {\scriptsize$\pm$0.42}} & \textcolor{blue}{89.98 {\scriptsize$\pm$2.02}} & - & 72.82 {\scriptsize$\pm$2.22} & 67.61 {\scriptsize$\pm$5.27} & - \\
KHGNN \cite{AAAI25-KHGNN} & 91.21 {\scriptsize$\pm$0.37} & \textcolor{blue}{74.80 {\scriptsize$\pm$1.10}} & \textcolor{orange}{88.47 {\scriptsize$\pm$0.47}} & \textcolor{orange}{89.60 {\scriptsize$\pm$1.64}} & \textcolor{orange}{92.55 {\scriptsize$\pm$1.03}} & \textcolor{orange}{74.31 {\scriptsize$\pm$2.87}} & \textcolor{blue}{71.23 {\scriptsize$\pm$3.94}} & 65.17 {\scriptsize$\pm$0.79} \\
\textbf{\texttt{HealHGNN} (Ours)} & \textcolor{red}{\textbf{91.98 {\scriptsize$\pm$0.31}}} & \textcolor{red}{\textbf{75.06} {\scriptsize$\pm$1.16}} & \textcolor{red}{\textbf{88.80 {\scriptsize$\pm$0.31}}} & \textcolor{red}{\textbf{90.16 {\scriptsize$\pm$1.78}}} & \textcolor{red}{\textbf{93.45 {\scriptsize$\pm$0.79}}} & \textcolor{red}{\textbf{77.18 {\scriptsize$\pm$2.35}}} & \textcolor{red}{\textbf{76.06 {\scriptsize$\pm$4.13}}} & \textcolor{red}{\textbf{68.22 {\scriptsize$\pm$0.56}}} \\

\bottomrule
\end{tabular}}

\end{table*}

\section{EXPERIMENTS}
\textbf{Objectives.} We evaluate our proposed method on a variety of hypergraph learning tasks and compare it with strong baselines. Specifically, we aim to address the following questions: \textbf{RQ1}: Can our model effectively handle both homogeneous and heterogeneous hypergraphs, as well as scale to large datasets? \textbf{RQ2}: Does our approach alleviate oversquashing in the hypergraph transfer setting and on the Long-Range Graph Benchmark(LRGB) , enabling efficient long-range information propagation? \textbf{RQ3}: Can our method mitigate oversmoothing when stacking deep layers, thereby preserving expressive node and hyperedge representations? \textbf{RQ4}: How critical are the proposed components to the overall effectiveness of our method? Codes are available in \textcolor{blue}{{\url{https://github.com/Mingzhang21/HealHGNN.}}} Further implementation notes are given in Appendix~\ref{reproductiion}. \\ 
\textbf{Experimental Settings.} 
We conduct experiments on ten publicly available datasets spanning various applications: Cora, Cora-CA, DBLP-CA, CiteSeer, and PubMed~\cite{nips19-hypergcn}, as well as NTU2012~\cite{NTU2012}, House~\cite{house}, Congress, Senate~\cite{congress-senate}, and Walmart~\cite{warlmart}.
Detailed descriptions of these baselines, along with the datasets, can be found in Appendix \ref{data-base}. All experiments are conducted on a machine equipped with an NVIDIA GeForce RTX 4090 GPU (24 GB memory), a 16-core Intel(R) Xeon(R) Gold 6430 CPU, and 120 GB RAM.

\begin{table}[h]
\centering
\vspace{-0.1cm}
\caption{Heterophily Experiments}
\label{tab:hetero}
\vspace{-0.2cm}
\resizebox{1.0\linewidth}{!}{
\begin{tabular}{l||ccccccc}
\toprule
 & \multicolumn{7}{c}{Heterophily Level}\\
Methods & 1 & 2 & 3 & 4 & 5 & 6 & 7 \\
\hline
HGNN               & 97.8 & 85.7 & 78.4 & 70.5 & 66.5 & 64.9 & 60.4 \\
HyperGCN           & 85.1 & 77.1 & 75.9 & 72.9 & 69.5 & 65.1 & 60.7 \\
HCHA               & 98.4 & 83.8 & 79.3 & 74.1 & 72.1 & 69.1 & 64.3 \\
\hline
AllDeepSets        & 99.8 & 87.8 & 82.4 & 77.3 & 73.4 & 71.8 & 68.4 \\
AllSetTransformer  & 99.9 & 86.5 & 83.3 & 78.5 & 74.2 & 72.4 & 69.2 \\
ED-HNN             & 99.9 & 91.3 & 88.4 & 85.1 & 81.7 & 77.3 & 76.5 \\
SensHNN            & 100  & 93.7 & 90.1 & 85.7 & 82.4 & 79.4 & 77.6 \\
\hline
\textbf{\texttt{HealHGNN}}               & \textbf{100}  & \textbf{94.1} & \textbf{90.8} & \textbf{87.1} & \textbf{84.3} & \textbf{82.2} & \textbf{79.8} \\
\bottomrule
\end{tabular}}
\vspace{-0.4cm}
\end{table}

\subsection{Hypergraph Node Classification (RQ1)}
\textbf{Setup.} We evaluate the performance of our model on the node classification task using eight real-world hypergraph datasets, and adopt classification accuracy (ACC) as the evaluation metric. These datasets span diverse domains, scales, and CE homophily levels, and are widely used benchmarks in hypergraph learning.  To ensure a fair comparison with baselines, we follow the experimental protocol of~\cite{ICLR22-edhnn}, where each dataset is randomly split into 50\% training, 25\% validation, and 25\% test nodes. Each case is repeated 10 times with different random splits, and we report the mean accuracy along with the standard deviation.\\
\textbf{Results.} Our model achieves competitive performance across all real-world datasets—matching strong baselines on homophilic hypergraphs and delivering clear gains on heterophilic ones. Notably, on the Senate dataset we outperform the second-best method by 5\%, underscoring the benefits of our heterophily-agnostic design.

\subsection{Synthetic Hypergraph under Heterophily}
\textbf{Setup.} The oversmoothing problem can hurt the performance of hypergraph networks, especially on heterophilic graphs. To investigate this, we evaluate hypergraph networks under varying heterophily. Following \cite{logc25-SensHNN} , we use synthetic hypergraphs generated with a contextual stochastic block model. Each hypergraph contains 5000 nodes divided into two classes and 1000 hyperedges, with the heterophily level controlled by adjusting the proportion of nodes from class 0.\\ 
\textbf{Results.} As observed in Table \ref{tab:hetero}, when the heterophily level is below three, our method performs slightly better than the other models. However, as the heterophily level increases, our approach demonstrates a significant improvement. This trend is reflected in Table \ref{tab:nctask}, where our method outperforms the baselines in heterophilic hypergraphs, showcasing its  ability to handle challenging heterophilic environments and mitigate the oversmoothing effect.

\subsection{Hypergraph Transfer (RQ2)}
\textbf{Setup.} Inspired by \cite{23osq}\cite{logc25-SensHNN}, we consider a hypergraph transfer task, where the objective is to transfer a label from a source node to a target node across hyperedges. Nodes are initialized with random-valued features, and we assign labels “1” and “0” to source and target nodes, respectively. We design four types of hypergraph distributions (see Figure~\ref{fig:transfertype} for a visual illustration): HyperEdgeSingle (HES($n$)) uses a single hyperedge of size $n$; HyperEdgePath (HEP($n$,$m$)) places the source in an 
$n$-node hyperedge and connects it to the target via a length-$m$ chain of degree-$2$ hyperedges; HyperEdgeRing (HER($n$,$m$)) forms a ring of $n$-node hyperedges, each sharing one node with the next, with the source–target distance controlled by $m$; and HyperEdgeDumbbell($n$,$m$) consists of two HER($n$,$m$) rings linked by a bridge of four degree-$2$ hyperedges. Importantly, the number of model layers is set equal to the required transfer distance.\\
\textbf{Results.} Figure~\ref{fig:transfer} reports the performance of different models on the hypergraph transfer tasks. Overall, HGNN\cite{AAAI19-hgnn} exhibits clear limitations: due to its reliance on standard hypergraph Laplacian propagation, node features tend to become homogenized over long distances, making it difficult to effectively transfer the label information from the source node to the target node. Similarly, the baseline methods struggle as the transfer distance increases, since their varying degrees of feature compression further hinder accurate long-range information transfer. In contrast, \texttt{HealHGNN} demonstrates consistently strong performance. On HES tasks, it maintains high accuracy, as nodes within the same hyperedge can still preserve discriminative information. More importantly, on HEP, HER, and HED, where the structures are more complex and the required propagation distance is longer (as illustrated in Fig~\ref{fig:transfertype}), thereby substantially increasing task difficulty, \texttt{HealHGNN} consistently outperforms all other models. This highlights its robustness and effectiveness in capturing both local and global structural dependencies.\\
\begin{figure}[ht]
    \vspace{-0.5cm}
    \centering
    \includegraphics[width=0.95\linewidth]{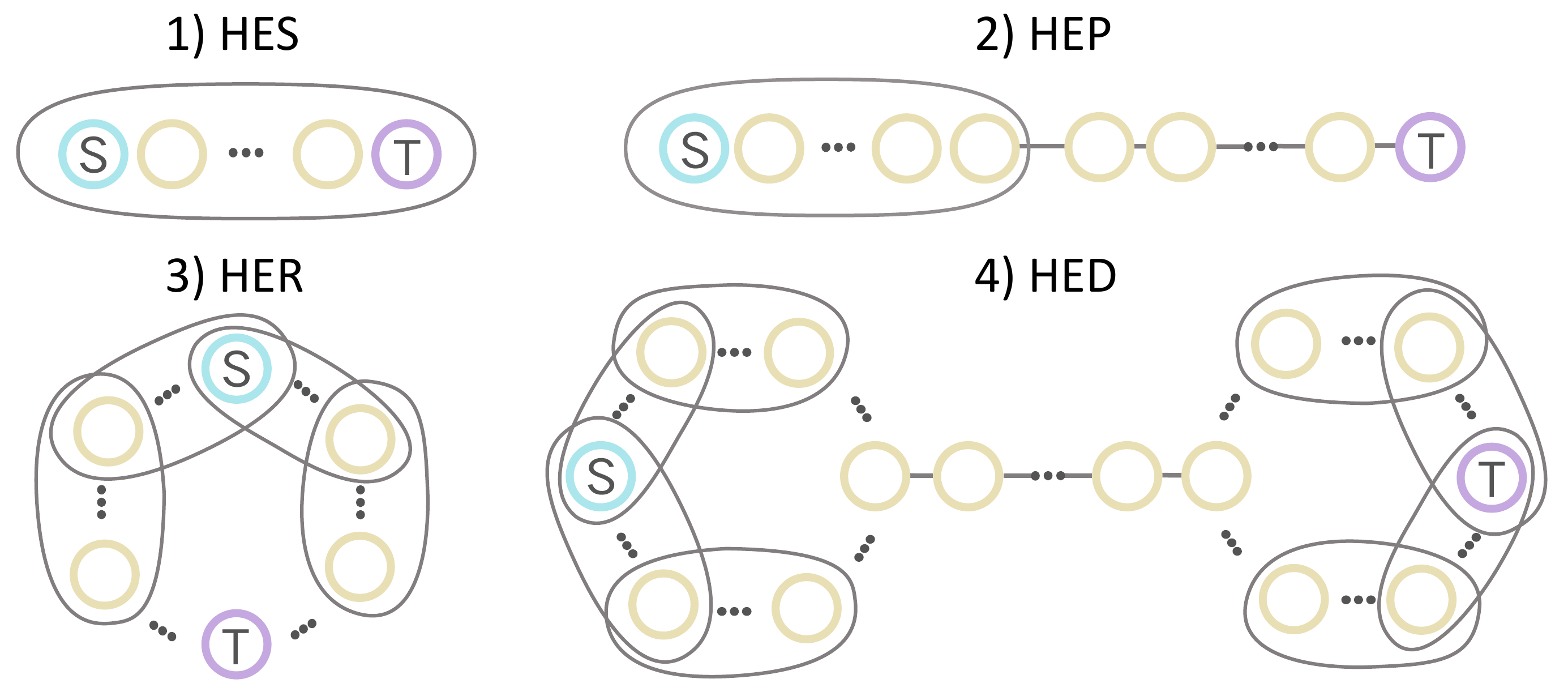}
         \vspace{-0.1cm} 
    \vspace{-0.3cm}
    \caption{Illustration of the Topological Structures}
            \label{fig:transfertype}
    \vspace{-0.3cm}
\end{figure}
 \vspace{-0.3cm}
\subsection{Long-Range Graph Benchmark (RQ2)}
\textbf{Setup.} We further examine the capacity of \texttt{HealHGNN} to model long-range dependencies using the Long-Range Graph Benchmark (LRGB)~\cite{LRGB}. In this benchmark, we consider two peptide datasets, Peptides-func and Peptides-struct, which are specifically designed to evaluate models on graph tasks requiring the capture of long-range interactions within molecular graphs. \\
\textbf{Results.} As shown in Figure~\ref{fig:lrgb}, graph transformer–based methods such as GT~\cite{20GT} and GRIT~\cite{GRIT} outperform classical MPNNs including GCN, HGNN~\cite{AAAI19-hgnn}, and GUMP~\cite{24GUMP}, owing to their ability to leverage global receptive fields. KHGNN~\cite{AAAI25-KHGNN} extends the receptive field by aggregating information along $k$-hop shortest paths, but this comes at the cost of computational overhead. In contrast, \texttt{HealHGNN} achieves superior performance while maintaining linear complexity $\mathcal{O}(|\mathcal{V}| + |\mathcal{E}|)$. By incorporating boundary awareness, it ensures that essential signals can be transferred efficiently across distant regions of the hypergraph.
\begin{figure}
    \centering
\includegraphics[width=\linewidth]{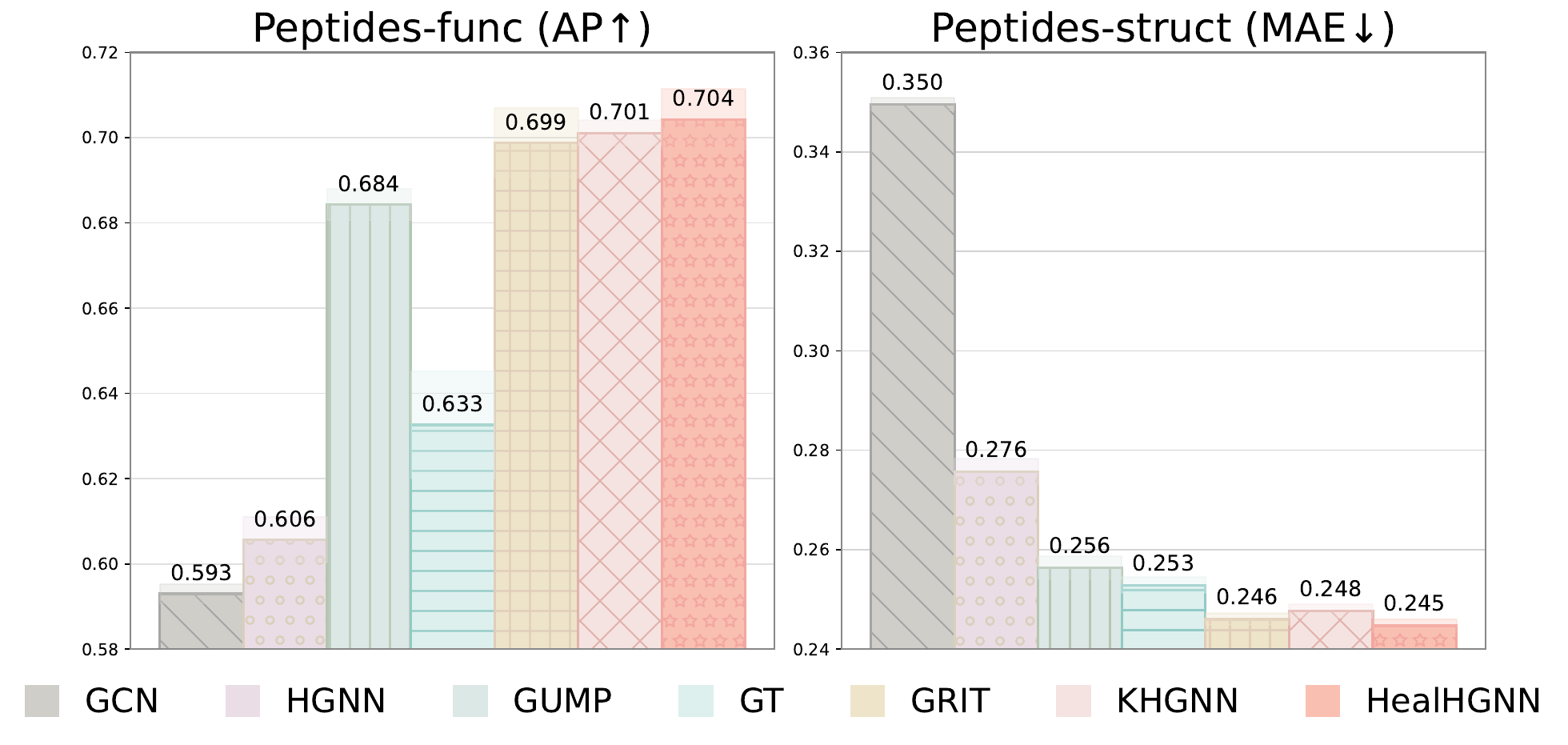}
    \vspace{-0.7cm}
    \caption{Performance comparison on LRGB tasks.}
            \label{fig:lrgb}
            \vspace{-0.65cm}
\end{figure}

\begin{figure*}[t]
    \centering
    \includegraphics[width=0.9\linewidth]{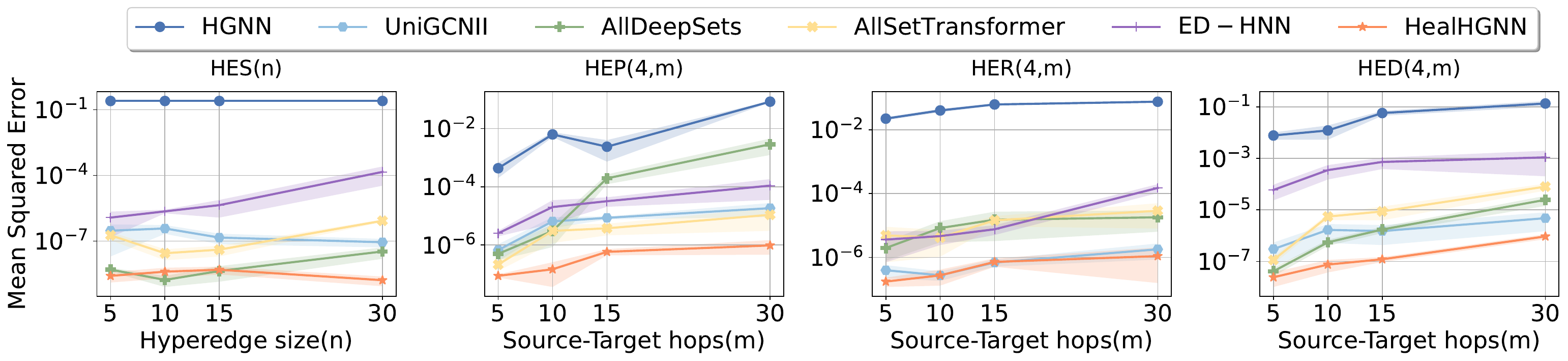}
         \vspace{-0.1cm} 
    \vspace{-0.3cm}
    \caption{Information Transfer Performance on the Four Hypergraph Types with Hyperedge Size 4.}
        \vspace{-0.6cm}
            \label{fig:transfer}
\end{figure*}

\begin{figure*}[t]
    \centering
    \vspace{0.3cm}
    \includegraphics[width=0.9\linewidth]{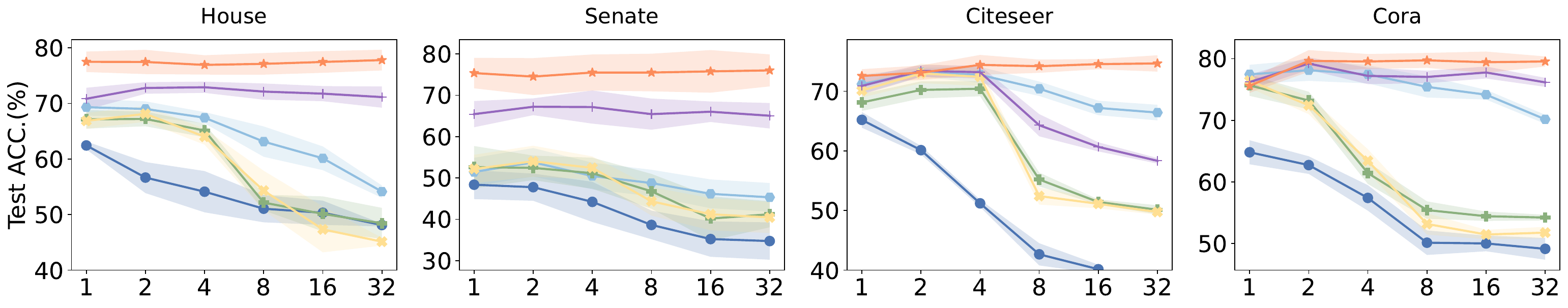}
         \vspace{-0.3cm} 
    \vspace{-0.1cm}
    \caption{Comparison of method performance across varying neural network depths on four datasets.}
        \vspace{-0.3cm}
            \label{fig:ringlayers}
\end{figure*}

\subsection{Neural Network Depth (RQ3)}
\textbf{Setup.} We compare \texttt{HealHGNN} against a range of representative baselines. To examine the effect of model depth on performance, we conduct experiments on four benchmark datasets: two heterophilic hypergraphs (House and Senate) and two homophilic hypergraphs (Citeseer and Cora). The number of layers is varied across $\{1, 2, 4, 8, 16, 32\}$ to study the relative performance of shallow versus deep architectures. The results are summarized in Figure~\ref{fig:ringlayers}.\\
\textbf{Results.} Overall, our method demonstrates stability as the number of layers increases, showing no significant performance fluctuations and even a slight upward trend. 
As discussed previously, the effect of oversmoothing becomes more severe in heterophilic hypergraphs, amplifying the performance gap in favor of \texttt{HealHGNN}.


\subsection{Ablation Study (RQ4)}

To evaluate the contribution of each component in \texttt{HealHGNN}, we create several variants by replacing the learnable coefficients $\gamma$, $\beta$, and hyperparameter $\mu$ with fixed constants ($\gamma_0$, $\beta_0$, $\mu_0$). Setting $\gamma_v=0$, $\beta_v=0$, or $\mu=0$ disables external inputs, boundary interactions, or hyperedge modulation, respectively, while $\gamma_v=\beta_v=\mu=0$ reduces \texttt{HealHGNN} to a standard HGNN.
\section{RELATED WORK}
\subsubsection*{\textbf{Hypergraph Representation Learning.}} Hypergraph learning has emerged as a powerful framework for modeling high-order relations that cannot be adequately captured by pairwise graphs. Foundational models such as HGNN~\cite{AAAI19-hgnn} and HyperGCN~\cite{nips19-hypergcn} established the basis for hypergraph neural networks by extending convolutional operators to the hypergraph setting. Subsequent approaches, including HNHN \cite{arxiv20-hnhn} and HCHA \cite{PR21-hcha}, enhanced representational capacity by incorporating nonlinear transformations and attention mechanisms. Recent works further advanced the field along theoretical, algorithmic, and efficiency dimensions. ORCHID~\cite{ICLR23-ricci} unified Ollivier–Ricci curvature on hypergraphs, providing new analytical tools. HyperGCL ~\cite{nips22-hyeprGCL} leveraged contrastive learning to improve generalization. To address scalability, LightHGNN~\cite{ICLR24-lightHGGNN} distilled HGNNs into lightweight MLPs, while TF-HNN~\cite{ICLR25-TF-HNN} reduced training cost via precomputed structural information.

\begin{table}[t]
    \centering 
    \caption{Ablation Study}
    \vspace{-0.35cm}
    \label{tab:model-comparison}
    \resizebox{0.45\textwidth}{!}{ 
     \begin{tabular}{l||cc||cc}
        \toprule
        \textbf{Model} & \textbf{Citeseer} & \textbf{Cora-CA} & \textbf{House} & \textbf{Senate} \\
        \hline
        \textbf{\texttt{HealHGNN}} & $\mathbf{74.56}{\scriptstyle\pm1.16}$ & $\mathbf{84.78}{\scriptstyle\pm1.57}$ & $\mathbf{77.18}{\scriptstyle\pm2.35}$ & $\mathbf{76.06}{\scriptstyle\pm4.13}$ \\
        $\gamma_0, \beta_0$, $\mu_0$  & $73.67{\scriptstyle\pm0.96}$ & $83.36{\scriptstyle\pm1.23}$ & $75.73{\scriptstyle\pm1.79}$ & $75.13{\scriptstyle\pm2.17}$ \\
        $\gamma_v =0$ &$71.15{\scriptstyle\pm1.58}$ & $82.11{\scriptstyle\pm1.34}$ & $72.27{\scriptstyle\pm2.34}$ & $71.19{\scriptstyle\pm3.16}$ \\
        $\beta_v =0$ & $66.46{\scriptstyle\pm1.17}$ & $81.41{\scriptstyle\pm1.18}$ &  $65.79{\scriptstyle\pm3.78}$ & $62.17{\scriptstyle\pm0.89}$ \\
        $\mu=0$ & $72.46{\scriptstyle\pm1.67}$ & $82.43{\scriptstyle\pm1.07}$ & $74.28{\scriptstyle\pm2.79}$ & $75.23{\scriptstyle\pm2.17}$ \\
           $\gamma_v, \beta_v, \mu =0$ &$63.04{\scriptstyle\pm1.72}$ & $79.16{\scriptstyle\pm1.96}$& $62.16{\scriptstyle\pm2.36}$ &  $50.89{\scriptstyle\pm3.36}$\\
        HGNN & $62.45{\scriptstyle\pm1.16}$ & $78.64{\scriptstyle\pm1.65}$ & $61.39{\scriptstyle\pm2.96}$ & $48.59{\scriptstyle\pm4.52}$ \\
        \bottomrule
    \end{tabular}
    }
    \vspace{-0.35cm}
\end{table}

\subsubsection*{\textbf{Oversmoothing and Oversquashing.}}
Deep hypergraph networks face two major challenges—oversmoothing and oversquashing—that constrain their expressive power. To alleviate oversmoothing, various architectures aim to preserve feature diversity during deep propagation. Deep-HGNN~\cite{arxiv22-deephgnn} employs high-order polynomial filters, while ED-HNN~\cite{ICLR22-edhnn} and UniGCNII~\cite{arxiv21-hgnn-uniGCNII} extend diffusion and message-passing to maintain discriminative representations. FrameHGNN~\cite{AAAI25-FrameHGNN}, SheafHyperGNN~\cite{NIPS23-sheafhypergnn}, PhenomNN~\cite{ICML24-PhenomNN}, and HDS~\cite{ICLR24-HDS} further introduce framelet transforms, sheaf Laplacians, or dynamic modeling to reinject node-specific signals and replenish Dirichlet energy. 
Methods such as SensHNN~\cite{logc25-SensHNN} and KHGNN~\cite{AAAI25-KHGNN} enhance message flow via direct interactions between distant nodes. Notably, these rewiring-oriented approaches may exacerbate oversmoothing, as denser pathways tend to homogenize node features, highlighting a key trade-off in hypergraph network design.

\subsubsection*{\textbf{Riemannian Geometry and Graph Neural Network.}} Euclidean space has long been the default of graph representation learning, but Riemannian manifolds have emerged as a compelling alternative~\cite{sun2019dna,sun2020perfect,sun2021heterogeneous,sun2021hyperbolic,sun2022self,sun2023deepricci,sun2024lsenet,sun2024spiking,sun2025deeper,sun2025riemanngfm,sun22,
nips18hnn,iclr23lantentGraphProduct}. The natural extension to \emph{hypergraphs} on manifolds is also gaining traction recently. \cite{25cluster} generates hypergraphs via sparse representation and reformulates orthogonality-constrained optimization as unconstrained Riemannian optimization on the Grassmann manifold. \cite{25sephere} embeds hypergraphs in hyperspherical manifolds to capture higher-order geometric relations through sphere-based message passing.
\section{CONCLUSION}
This paper presents a heterophily-agnostic hypergraph neural network (\texttt{HealHGNN}) through Riemannian heat flow. We start with the theoretical discovery that long-range dependence is related to hypergraph bottlenecks, and then establish the Adaptive Local Exchanger that locally adjusts subhypergraph bottlenecks in the functional space, creating a new message passing mechanism with theoretical guarantees. Consequently, we formulate \texttt{HealHGNN} that bidirectionally propagates messages from nodes and hyperedges with linear complexity. Extensive experiments across multiple benchmarks demonstrate that HealHGNN consistently improves performance over competitive baselines.
\begin{acks}
This work is supported in part by NSFC under grants 62202164, U25B2029, 62072052 and 62322202, and the National Key Research and Development Program of China (2024YFF0907401). Philip S. Yu is supported in part by NSF under grants III-2106758, and POSE-2346158.
\end{acks}
\newpage
\newpage

\bibliographystyle{ACM-Reference-Format}
\bibliography{www26}
\appendix

\section{GLOSSARY OF NOTATIONS}
\label{notation}
We provide a summary of the notation employed in this paper, along with their respective descriptions, in Table~\ref{tab:notation}.
\begin{table}[h]
\centering
\vspace{-0.15in}
\caption{Notations.}
\vspace{-0.15in}
\label{tab:notation}
\resizebox{\linewidth}{!}{
\begin{tabular}{|c|c|}
\toprule
\textbf{Notation} & \textbf{Description}   \\
\midrule
$\mathcal{M}$  & A Riemannian manifold\\
$h_\mathcal{M}$ & The Cheeger's constant of $\mathcal{M}$ \\
$\xi$ & A small value variable \\
$\varepsilon(\cdot)$ & The radius function \\
  $\lambda_i, \psi_i$ & The $i$-th eigenvalue and eigenfunction of the Laplace-Beltrami operator.\\
  $\partial_t$ & The parital derivative w.r.t. $t$\\ 
  $u(t, \boldsymbol{x})$  & The value of the heat equation solution\\
  $\frac{\partial}{\partial n}$ & The derivative in the direction of the outward normal \\
  $K(t)$ & The heat kernel\\
   $\lambda$ & The spectral gap of a closed manifold\\
\midrule
     $\mathcal{G}$ & A hypergraph with node set $\mathcal{V}$, edge set $\mathcal{E}$ and incidence matrix \textbf{H}\\
     $\mathcal{L}$ &  The normalized hypergraph Laplacian \\
    $\varphi$ & A multi-layer perceptron (MLP) \\ 
     $Z$ & The initial condition or function for the heat equation \\
    $\mathbb{I}$ & The indicator function \\
   $E(X)$ & The Dirichlet energy of feature matrix $X$ \\
   
$\alpha, \beta$ & The coefficients of the Robin condition \\
$S$ & The interior node set of hypergraph $\mathcal{G}$ \\
$S_e$ &The interior edge set of hypergraph $\mathcal{G}$ \\
$\partial_v S$ & The boundary node set of hypergraph $\mathcal{G}$ \\
   $\partial_e S$ & The boundary edge set of hypergraph $\mathcal{G}$ \\
\midrule
   $\mathbf{\Gamma}_{\partial \mathcal{S}}$ & The source term of nonhomogeneous Robin condition for boundary sets \\
   $\mathbf{F}_\mathcal{S}$ & The steady state of the heat equation for interior sets\\
   $\mathbf{\Gamma}$ & The concatenation of $\mathbf{F}_\mathcal{S}$ and $\mathbf{\Gamma}_{\partial \mathcal{S}}$\\
  \bottomrule
\end{tabular}}
\end{table}

\vspace{-0.15in}
\section{Theorems and Proofs}
This section provides detailed proofs of the theorems and the derivation of the heat-equation solution.
\vspace{-0.1in}
\subsection{Proofs of Theorem \ref{bdcondition}}\label{prf:flexibility of rc}
\textbf{Theorem \ref{bdcondition} (Flexibility of the Robin Condition).} Consider a thin tube–like region $S$ of length $L$, whose cross–section at position $s \in [0,L]$ is a small ball of radius $\varepsilon(s)$. The boundary of this region is given by $\partial S = \{0,L\} \times \partial \mathcal{B}_{\varepsilon(s)}$. Let $\lambda_D$, $\lambda_N$, and $\lambda_{R}$ denote the spectral gaps of the Laplace operator under Dirichlet, Neumann, and Robin boundary conditions, respectively.
When the radius is constant, i.e., $\varepsilon(s) = \varepsilon_0$, the spectral gaps satisfy the ordering
$\lambda_D \;\geq\; \lambda_{R} \;\geq\; \lambda_N$.

\begin{proof}
Let $S$ be a thin tube-like region of length $L$, whose cross-section at position $s \in [0,L]$ is a small ball of radius $\varepsilon(s)$.  Assume that $\varepsilon(s) : [0,L] \to (0,\delta)$ with $L \gg \delta$ for a small $\delta>0$.
First, we recall the Cheeger and Buser's  Inequality for a  closed manifold $\mathcal{M}$\cite{1982A}: if the Ricci curvature of $\mathcal{M}$ satisfies
\begin{equation}
\mathrm{Ric}_M \ge -(n-1)\,\xi^2, \quad \xi \ge 0,    
\end{equation}
then the spectral gap $\lambda$ of the Laplace–Beltrami operator satisfies
\begin{equation}
\frac{h_\mathcal{M}^2}{4} \le \lambda < C(\xi \cdot h_\mathcal{M} + h_\mathcal{M}^2),    
\end{equation}
where $h_\mathcal{M}$ is the Cheeger constant and $C$ is a constant. This inequality shows that the spectral gap is closely related to the “bottleneck” of the manifold. Taking $S$ as a submanifold embedded in a closed manifold, the spectral gap $\lambda \sim \inf \varepsilon^{n-1}$. For the thin tube $S$, when the cross-section radius is constant, i.e., $\varepsilon(s) = \varepsilon_0$, the effect of boundary conditions on the spectral gap can be analyzed through the behavior of the heat kernel. Under the Dirichlet condition, the heat kernel decays at a rate governed by $\exp(-\varepsilon_0^{-2} t)$, reflecting the strong suppression of oscillations at the boundary. In contrast, under the Neumann condition, the decay is determined by $\exp(-L^{-2} t)$, which corresponds to much slower attenuation along the longitudinal direction of the tube.    
Since $\varepsilon_0$ is small, we have $\varepsilon_0^{-2}$ large and $\varepsilon_0^{\,n-1}$ small. Combining these scaling observations with Cheeger-type estimates, it follows that the spectral gaps satisfy the ordering
\begin{equation}
\lambda_D \;\ge\; \lambda_R \;\ge\; \lambda_N.
\end{equation}
\end{proof}
\subsection{Proofs of Theorem \ref{smooth}}\label{prf:Spectral Gap and Dirichlet Energy}
\textbf{Theorem \ref{smooth} (Spectral Gap and Dirichlet Energy).} Let $0 = \lambda_0 < \lambda_1 \le \lambda_2 \le \dots$ be the eigenvalues of the hypergraph Laplacian $\mathcal{L}$, then the Dirichlet Energy $E(x):=x^\top \mathcal{L} x$ satisfies
\begin{equation}
E(u_k(t)) 
\le e^{-2\lambda_1 (t - t_k)} E(Z_k).    
\end{equation}
In particular, the Dirichlet energy decays at least exponentially with rate $2\lambda_1$, where $\lambda_1$ is the spectral gap.

\begin{proof}
Since $\mathcal{L}$ is symmetric, there exists an orthonormal eigenbasis $\{\psi_i\}$ with $\mathcal{L}\psi_i=\lambda_i \psi_i$.
We expand the initial state as
\begin{equation}
Z_k = \sum_i c_i \psi_i, \qquad c_i := \psi_i^\top Z_k.    
\end{equation}
The closed-form solution of the diffusion equation is
\begin{equation}
u_k(t) = e^{-(t-t_k)\mathcal{L}}Z_k = \sum_i c_i e^{-\lambda_i (t-t_k)} \psi_i .   
\end{equation}
Now consider the Dirichlet energy:
\begin{equation}
\begin{aligned}
E(u_k(t)) 
&= u_k(t)^\top \mathcal{L} u_k(t) \\
&= \Big(\sum_i c_i e^{-\lambda_i \Delta t} v_i\Big)^\top 
   \Big(\sum_j \lambda_j c_j e^{-\lambda_j \Delta t} v_j\Big), 
\end{aligned}    
\end{equation}
where $\Delta t := t-t_k \ge 0$.
Using orthonormality $(\psi_i^\top \psi_j=\delta_{ij})$, this reduces to
\begin{equation}
E(u_k(t)) = \sum_i \lambda_i |c_i|^2 e^{-2\lambda_i \Delta t}.
\end{equation}
Since for all $i\ge 1$, we have $\lambda_i \ge \lambda_1$, it follows that
\begin{equation}
E(u_k(t)) \le e^{-2\lambda_1 \Delta t} \sum_{i\ge 1} \lambda_i |c_i|^2,    
\end{equation}
The $i=0$ term vanishes because $\lambda_0=0$. Hence
\begin{equation}
E(u_k(t)) \le e^{-2\lambda_1 (t-t_k)} E(Z_k).    
\end{equation}
This shows that the Dirichlet energy decays at least exponentially with rate $2\lambda_1$, where $\lambda_1$ is the spectral gap of the Laplacian.
\end{proof}

\subsection{Proofs of Theorem \ref{source}}\label{prf:source}
\textbf{Theorem \ref{source} (Mode Preservation under Source Energy Injection).} Consider the diffusion equation with source term $\partial_t u_k(t,\mathbf{x}) = -\mathcal{L}\,u_k(t,\mathbf{x}) + f(t,\mathbf{x})$. Expanding in the Laplacian eigenbasis gives $u(t) = \sum_i \hat u_i(t)\psi_i$, $f(t) = \sum_i \hat f_i(t)\psi_i$ with coefficients $\hat u_i(t) = \langle u(t), \psi_i \rangle$, $\hat f_i(t) = \langle f(t), \psi_i \rangle$, and each mode satisfies $\dot{\hat u}_i(t) = -\lambda_i \hat u_i(t)\\ + \hat f_i(t)$. If the source persistently excites mode $i$ over a time step $\Delta t = t_{k+1} - t_k$, i.e., $\hat f_i(s) \ge \eta_i > 0$ for all $s \in [t_k, t_{k+1}]$, then its amplitude is bounded below by
\begin{align}
\hat u_i(t_{k+1}) \ge \frac{\eta_i}{\lambda_i} (1 -e^{-\lambda_i \Delta t}) + e^{-\lambda_i \Delta t} \hat u_i(t_k).    
\end{align}
showing that persistent energy injection prevents the mode from vanishing.
\begin{proof}
Let $\mathcal{L}$ be a symmetric hypergraph Laplacian with eigenpairs $(\lambda_i, \psi_i)$. Expanding the node features $u(t)$ and source term $f(t)$ in the Laplacian eigenbasis, we have
\begin{align}
u(t) &= \sum_i \hat u_i(t) \psi_i, & 
f(t) &= \sum_i \hat f_i(t) \psi_i, \\
\hat u_i(t) &= \langle u(t), \psi_i \rangle, & 
\hat f_i(t) &= \langle f(t), \psi_i \rangle.
\end{align}
Projecting the diffusion equation $\partial_t u(t) = -\mathcal{L} u(t) + f(t)$ onto $\psi_i$ yields a decoupled ODE for each mode
\begin{align}
\frac{d}{dt} \hat u_i(t) = -\lambda_i \hat u_i(t) + \hat f_i(t), \quad \hat u_i(0) = \langle u_0, \psi_i \rangle.    
\end{align}
Considering the discrete time step $\Delta t = t_{k+1} - t_k$, the solution is given by Duhamel’s formula:
\begin{align}
\hat u_i(t_{k+1}) = e^{-\lambda_i \Delta t} \hat u_i(t_k) + \int_{t_k}^{t_{k+1}} e^{-\lambda_i (t_{k+1}-s)} \hat f_i(s)\, ds.    
\end{align}
Assume that the source injects persistent energy into mode $i$ during this interval, i.e., $\hat f_i(s) \ge \eta_i > 0$ for all $s \in [t_k, t_{k+1}]$. Then the integral term can be lower-bounded as
\begin{align}
\int_{t_k}^{t_{k+1}} e^{-\lambda_i (t_{k+1}-s)} \hat f_i(s)\, ds& \ge \eta_i \int_{t_k}^{t_{k+1}} e^{-\lambda_i (t_{k+1}-s)} ds\\ &= \frac{\eta_i}{\lambda_i} (1 - e^{-\lambda_i \Delta t}).    
\end{align}
Combining this with the decaying contribution from the previous step, we obtain the lower bound for the mode amplitude at $t_{k+1}$:
\begin{equation}
\hat u_i(t_{k+1}) \ge \frac{\eta_i}{\lambda_i} (1 - e^{-\lambda_i \Delta t}) + e^{-\lambda_i \Delta t} \hat u_i(t_k).      
\end{equation}
This shows that as long as the source injects persistent energy during each time step, the amplitude of the mode cannot vanish completely and thereby alleviates oversmoothing.
\end{proof}

\subsection{Proofs of Equivalence}
\label{equivalence}
Let $w(t,\mathbf{x})$ satisfy $h\!\circ w=\gamma$ on $\partial_v\mathcal{S}$. Define the lifted unknown $v:=u-w$. By linearity,
\begin{equation}
 \partial_t v=\partial_t u-\partial_t w,\qquad -\mathcal{L}v=-\mathcal{L}u+\mathcal{L}w.   
\end{equation}
Assume the original problem is
\begin{align}
\begin{cases}
\partial_t u(t,\mathbf{x})=-\mathcal{L}u(t,\mathbf{x}),&\mathbf{x}\in \mathcal{S},\\[4pt]
h(u)=\gamma\ , &\mathbf{x} \in\partial_{v}\mathcal{S},\\[4pt]
u(0,\mathbf{x})=Z(\mathbf{x}).
\end{cases}    
\end{align}
Substituting $u=v+w$ gives
\begin{equation}
\partial_t v=-\mathcal{L}v+\underbrace{\big(-\mathcal{L}w-\partial_t w\big)}_{:=\,\Gamma(w)},  
\end{equation}
and the boundary becomes $h(v)=h(u)-h(w)=\gamma-\gamma=0$ on $\partial_v\mathcal{S}$; the initial condition is $v(0,\mathbf{x})=Z(\mathbf{x})-w(0,\mathbf{x})$.
Conversely, if $v$ solves
\begin{equation}
 \begin{cases}
\partial_t v(t,\mathbf{x})=-\mathcal{L}v(t,\mathbf{x})+\Gamma(w),\\[4pt] h(v)=0, 
\\[4pt] v(0,\mathbf{x})=Z(\mathbf{x})-w(0,\mathbf{x}).
\end{cases}   
\end{equation}\\
then setting $u:=v+w$ yields $\partial_t u=-\mathcal{L}u$, $h(u)=\gamma$ on $\partial_v\mathcal{S}$, and $u(0,\mathbf{x})=Z(\mathbf{x})$. Therefore the two formulations are equivalent.

\subsection{Derivation of the Heat Equation Solution}
\label{solution}
Recall the coupled update rules:
\begin{equation}
\scalebox{0.68}{
$
\begin{cases}
\mathcal{L}_{\mathcal{V}}^{S\to S} X_{S} + \mathcal{L}_{\mathcal{V}}^{\partial_v S\to S} X_{\partial_{v}S}
- \lambda \Big( \mathcal{P}^{S_e\to S}_{\mathcal{E}\to \mathcal{V}} X_{S_{e}} + \mathcal{P}^{\partial_e S\to S}_{\mathcal{E}\to \mathcal{V}} X_{\partial_{e}{S}} \Big) = F_\mathcal{S}, \\[1ex]
\mathrm{diag}(\alpha_v) X_{\partial_{v}{S}} + \mathrm{diag}(\beta_v) \mathcal{L}_{\mathcal{V}}^{S\to \partial_v S} X_{S}
- \lambda \Big( \mathcal{P}^{S_e\to \partial_v S}_{\mathcal{E}\to \mathcal{V}} X_{S_{e}} + \mathcal{P}^{\partial_e S\to \partial_v S}_{\mathcal{E}\to \mathcal{V}} X_{\partial_{e}S} \Big) = \Gamma_{\partial_{v}{S}}, \\[1ex]
\mathcal{L}_{\mathcal{E}}^{S_e\to S_e} X_{S_{e}} +\mathcal{L}_{\mathcal{E}}^{\partial_e S\to S_{e}} X_{\partial_{e}{S}}
- \mu \Big( \mathcal{P}^{S\to S_e}_{\mathcal{V}\to \mathcal{E}} X_{S} + \mathcal{P}^{\partial_v S\to S_e}_{\mathcal{V}\to \mathcal{E}} X_{\partial_{v}{S}} \Big) = F_{S_{e}}, \\[1ex]
\mathrm{diag}(\alpha_e) X_{\partial_{e}{S}} + \mathrm{diag}(\beta_e) \mathcal{L}_{\mathcal{E}}^{S_e \to \partial_e S} X_{S_{e}}
- \mu \Big( \mathcal{P}^{S\to \partial_e S}_{\mathcal{V}\to \mathcal{E}} X_{S} + \mathcal{P}^{\partial_v S\to \partial_e S}_{\mathcal{V}\to \mathcal{E}} X_{\partial_{v}{S}} \Big) = \Gamma_{\partial_{e}{S}}.
\end{cases}
$
} 
\end{equation}
Then it can be expressed as a compact block-matrix form, which is convenient for unified analysis and computation.
\begin{align}
A
\begin{bmatrix}
X_V\\[2pt] X_E
\end{bmatrix}
=
\begin{bmatrix}
\Gamma_V\\[2pt] \Gamma_E
\end{bmatrix},\quad  
X=\begin{bmatrix}
X_{S}\\[2pt] X_{\partial{S}} \\
\end{bmatrix},\quad
\Gamma =\begin{bmatrix}
F_{S}\\[2pt] \Gamma_{\partial{S}}
\end{bmatrix}.
\end{align}
For the heat diffusion system, let $\mathbf{D}$, $\mathbf{U}$, and $\mathbf{V}$ denote the diagonal, upper, and lower blocks of the system matrix $\mathbf{A}$, respectively. 
At the $k$-th iteration, the learnable diagonal terms are updated as
\begin{equation}
\scalebox{0.8}{
$
\operatorname{diag}(\alpha_v) = \alpha(\mathbf{X}_{\mathcal{V}}^{(k)}; \theta_{\alpha}),\quad
\operatorname{diag}(\beta_v) = \beta(\mathbf{X}_{\mathcal{V}}^{(k)}; \theta_{\beta}),\quad
\mathbf{\Gamma} = \gamma(\mathbf{X}_{\mathcal{V}}^{(k)}; \theta_{\gamma}).
$}
\end{equation}
\begin{equation}
\scalebox{0.8}{
$
\operatorname{diag}(\alpha_e) = \alpha(\mathbf{X}_{\mathcal{E}}^{(k)}; \theta_{\alpha}),\quad
\operatorname{diag}(\beta_e) = \beta(\mathbf{X}_{\mathcal{E}}^{(k)}; \theta_{\beta}),\quad
\mathbf{\Gamma} = \gamma(\mathbf{X}_{\mathcal{E}}^{(k)}; \theta_{\gamma}).
$}
\end{equation}
The feature update rule is then given by the block iteration scheme
\begin{equation}
\mathbf{X}^{(k+1)}=\mathbf{D}^{-1}(\mathbf{V}+\mathbf{U})\mathbf{X}^{(k)}+\mathbf{D}^{-1}\mathbf{\Gamma}    
\end{equation}
Accordingly, the update of node $i$ at iteration $k+1$ takes the form
\begin{align}
\label{nodeupdate}
 x_i^{(k+1)}
= \underbrace{\sum_{j:\, j \sim i} (\mathcal{P}_{\mathcal{V} \to \mathcal{V}})_{ij} \, x_j^{(k)}}_{\text{node-to-node aggregation}}
+ \underbrace{\sum_{e:\, e \sim i} (\mathcal{P}_{\mathcal{E} \to \mathcal{V}})_{ei} \, x_e^{(k)}}_{\text{edge-to-node aggregation}}
+ \gamma_i,   
\end{align}
where $j \sim i$ denotes that $i$ and $j$ share at least one hyperedge, and $e \sim i$ means that node $i$ belongs to hyperedge $e$.
The entries of the propagation operators are specified as follows. For node-to-node aggregation,
\begin{equation}
(\mathcal{P}_{\mathcal{V}\to \mathcal{V}})_{ij} =
\begin{cases}
\dfrac{2}{\sqrt{d_i d_j}}\displaystyle\sum_{e:\,i,j\in e}\dfrac{1}{d_e} &
i\sim j,\ i\in \mathcal{S},\ j\in \mathcal{S},\\[8pt] 
\dfrac{1}{\sqrt{d_i d_j}}\displaystyle\sum_{e:\,i,j\in e}\dfrac{1}{d_e} &
i\sim j,\ i\in \mathcal{S},\ j\in \partial{_{v}\mathcal{S}},\\[8pt] 
\dfrac{\beta_v}{\alpha_v}\dfrac{1}{\sqrt{d_i d_j}}\displaystyle\sum_{e:\,i,j\in e}\dfrac{1}{d_e} &
i\sim j,\ i\in \partial_{v}{\mathcal{S}},\ j\in \mathcal{S},\\[8pt] 
0&\text{otherwise.}
\end{cases}    
\end{equation}
whereas for edge-to-node aggregation,
\begin{equation}
(\mathcal{P}_{\mathcal{E}\to \mathcal{V}})_{ij} =
\begin{cases}
\dfrac{\mu}{1}\dfrac{1}{\sqrt{d_i d_j}} &
i\sim j,\ i\in \mathcal{S},\ j\in \mathcal{S}_{e},\\[8pt] 
\dfrac{\mu}{\alpha_{v}}\dfrac{1}{\sqrt{d_i d_j}} &
i\sim j,\ i\in \partial{_{v}\mathcal{S}},\ j\in \mathcal{S}_{e},\\[8pt] 
0&\text{otherwise.}
\end{cases}    
\end{equation}
For compactness, define $S_{ij} := \sum_{e:\, i,j \in e} \frac{1}{d_e}$ as the total inverse degree of hyperedges shared by nodes $i$ and $j$, where $d_e$ is the degree of hyperedge $e$, and $p_i=\frac{\beta_v}{\alpha_v}$. We define the indicator function $I_i := \mathbb{I}(i \in \mathcal{S})$, which equals $1$ if node $i$ belongs to the interior set $\mathcal{S}$ and $0$ otherwise.
Then the node-to-node propogation operator can be written as
\begin{equation}
\mathcal{P}_{\mathcal{V} \to \mathcal{V}}(i,j)
= \frac{1}{\sqrt{d_i d_j}} \, S_{ij} \,
\Big( I_i I_j \,+\, I_i \,+\, p_i (1-I_i) I_j \Big),    
\end{equation}
The edge-to-node propogation operator  is
\begin{equation}
P_{\mathcal{E} \to \mathcal{V}}(i,e)
= \mu \; \frac{1}{\sqrt{d_i d_e}}
\Big( I_i \,+\, \frac{1}{\alpha_v} (1-I_i) \Big),    
\end{equation}
where $\mu$ is a scaling coefficient, and $\alpha_v$ regulates the relative influence of hyperedges on boundary versus interior nodes.
Analogous to the node update, the feature update for hyperedge $e$ at iteration $k+1$ can be written as
\begin{align}
\label{edgeupdate}
x_e^{(k+1)}
= \underbrace{\sum_{f:\, f \sim e} (\mathcal{P}_{\mathcal{E} \to \mathcal{E}})_{ef} \, x_f^{(k)}}_{\text{edge-to-edge aggregation}}
+ \underbrace{\sum_{i:\, i \sim e} (\mathcal{P}_{\mathcal{V} \to \mathcal{E}})_{ie} \, x_i^{(k)}}_{\text{node-to-edge aggregation}}
+ \gamma_e,    
\end{align}
where $f \sim e$ indicates that hyperedges $e$ and $f$ share at least one common node, and $i \sim e$ indicates that node $i$ belongs to hyperedge $e$.
The entries of the edge-to-edge propagation operator are defined by
\begin{equation}
(\mathcal{P}_{\mathcal{E}\to \mathcal{E}})_{ef} =
\begin{cases}
\dfrac{2}{\sqrt{d_e d_f}} \displaystyle\sum_{i:\, i\in e\cap f}\dfrac{1}{d_i} & e\sim f,\ e\in \mathcal{S}_e,\ f\in \mathcal{S}_e,\\[8pt]
\dfrac{1}{\sqrt{d_e d_f}} \displaystyle\sum_{i:\, i\in e\cap f}\dfrac{1}{d_i} & e\sim f,\ e\in \mathcal{S}_e,\ f\in \partial_{e}\mathcal{S},\\[8pt]
\dfrac{\beta_e}{\alpha_e}\dfrac{1}{\sqrt{d_e d_f}}\displaystyle\sum_{i:\, i\in e\cap f}\dfrac{1}{d_i} & e\sim f,\ e\in \partial_{e}\mathcal{S},\ f\in \mathcal{S}_e,\\[8pt]
0 & \text{otherwise}.
\end{cases}    
\end{equation}
For compactness, define
\begin{equation}
S_{ef} := \sum_{i:\, i \in e\cap f} \frac{1}{d_i},
\qquad
q_e := \frac{\beta_e}{\alpha_e},
\qquad
J_e := \mathbb{I}(e \in \mathcal{S}_e).    
\end{equation}
Then the edge-to-edge propagation operator can be expressed as
\begin{equation}
\mathcal{P}_{\mathcal{E} \to \mathcal{E}}(e,f)
= \frac{1}{\sqrt{d_e d_f}} \, S_{ef}\,
\Big( J_e J_f + J_e + q_e (1-J_e) J_f \Big),    
\end{equation}
Similarly, the contribution from nodes to hyperedge $e$ is given by
\begin{equation}
(P_{\mathcal{V}\to \mathcal{E}})_{ie} =
\begin{cases}
\dfrac{\lambda}{1}\dfrac{1}{\sqrt{d_i d_e}}, & i \in e,\ e\in \mathcal{S}_e,\\[6pt]
\dfrac{\lambda}{\alpha_e}\dfrac{1}{\sqrt{d_i d_e}}, & i \in e,\ e\in \partial_{e}\mathcal{S},\\[6pt]
0, & \text{otherwise}.
\end{cases}    
\end{equation}
Compact form:
\begin{equation}
\mathcal{P}_{\mathcal{V} \to \mathcal{E}}(i,e) = \lambda\frac{1}{{\sqrt{d_i d_e}}}\,
\big( J_e + \frac{1}{\alpha_e}(1-J_e) \big).    
\end{equation}

\section{Datasets and Baselines}
\label{data-base}
\subsection{Datasets}
To demonstrate the generality of our approach, we evaluate it on benchmark datasets spanning different domains and structural types.


\textbf{Cora/DBLP-CA} are co-authorship graphs where nodes are authors and hyperedges are joint publications (Cora-CA from the Cora citation network; DBLP-CA from the DBLP database). \textbf{Citeseer} and \textbf{Pubmed} are co-citation hypergraphs of scientific articles, linking papers that are cited together. \textbf{Congress} models US legislative collaboration with nodes as members and hyperedges as bills connecting sponsors and co-sponsors across chambers. \textbf{Senate} uses the same bill–co-sponsor construction for the Senate, with party labels on nodes. \textbf{House} links representatives who serve on the same committee, with nodes labeled by party.



\subsection{Baselines}
\begin{itemize}
    \item \textbf{ED-HNN} \cite{ICLR22-edhnn} approximates continuous equivariant hypergraph diffusion operators and combines hypergraph star expansions with standard message passing.
    \item \textbf{SheafHyperGNN} \cite{NIPS23-sheafhypergnn} incorporate cellular sheaves into hypergraphs, introducing linear and nonlinear sheaf Laplacians.
    \item  \textbf{PhenomNN} \cite{ICML24-PhenomNN} learns node embeddings as minimizers of hypergraph-regularized energies, trained end-to-end via supervised bilevel optimization.
    \item \textbf{HDS} \cite{ICLR24-HDS} formulates HGNNs as ODE-driven dynamics with control–diffusion steps.
    \item  \textbf{SensHNN} \cite{logc25-SensHNN} mitigates oversquashing by k-hop message passing that permits distant vertices to interact while retaining local cues..
    \item  \textbf{FrameHGNN} \cite{AAAI25-FrameHGNN} uses framelet-based hypergraph convolutions with low/high-pass filters plus residual/identity mappings to curb oversmoothing.
    \item \textbf{KHGNN} \cite{AAAI25-KHGNN} aggregates along k-hop shortest paths between nodes and hyperedges via bisection nested convolutions (on HyperGINE) to model long-range structure.
    \item  \textbf{TF-HNN} \cite{ICLR25-TF-HNN} introduces a training-free message passing module (TF-MP-Module) that precomputes hypergraph structural information during preprocessing. 
\end{itemize}
\section{REPRODUCIBILITY STATEMENT}
\label{reproductiion}
In order to guarantee the reproducibility of our findings, we present a comprehensive list of the hyperparameter settings employed during our experiments. For all datasets, we employed the GeLU activation function, kept the coupling coefficient to be $1$ and utilized LayerNorm. The specific configurations used for our model in the hypergraph node classification task are provided in Table~\ref{table:hyperparameter}.
\vspace{-0.1cm}
\begin{table}[h]
\centering
\vspace{-0.1cm}
\caption{Hyperparameter settings for node classification task.}
\vspace{-0.2cm}
\label{table:hyperparameter}
\resizebox{0.95\linewidth}{!}{
\begin{tabular}{lccccccccc}
\toprule
\textbf{Dataset} & \textbf{n\_layer} & \textbf{hid\_dim} & \textbf{tau} & \textbf{dropout} & \textbf{lr} & \textbf{w\_decay} \\
\midrule
Cora-CA & 3 & 128 & 0.001 & 0.7 & 5e-05 & 5e-07 \\
DBLP-CA & 5 & 128 & 0.01 & 0.7 & 1e-03 & 0.0 \\
Citeseer & 3 & 512 & 0.1 & 0.8 & 5e-03 & 5e-05 \\
Pubmed & 3 & 64 & 0.5 & 0.6 & 5e-03 & 5e-05 \\
Congress & 5 & 512 & 0.1 & 0.35 & 5e-03 & 0.0 \\
House & 5 & 512 & 0.5 & 0.35 & 3e-04 & 0.0 \\
Senate & 5 & 512 & 0.5 & 0.35 & 3e-04 & 0.0 \\
NTU2012 & 3 & 512 & 0.5 & 0.3 & 3e-05 & 0.0 \\
Walmart & 5 & 256 & 0.5 & 0.35 & 3e-04 & 0.0 \\
\bottomrule
\end{tabular}
}
\end{table}

\end{document}